\documentclass{article} 
\usepackage{iclr2026_workshop,times}

\usepackage{amsmath,amsfonts,bm}

\def\eqref#1{equation~\ref{#1}}

\def\1{\bm{1}}

\DeclareMathAlphabet{\mathsfit}{\encodingdefault}{\sfdefault}{m}{sl}
\SetMathAlphabet{\mathsfit}{bold}{\encodingdefault}{\sfdefault}{bx}{n}

\usepackage{hyperref}
\usepackage{url}
\usepackage{algorithm}   
\usepackage{algorithmic} 
\usepackage{microtype}
\usepackage{graphicx}
\usepackage{subfigure}
\usepackage{booktabs} 
\usepackage{multirow}
\usepackage{colortbl}      
\usepackage{xcolor}        %
\usepackage{amsthm}
\iclrfinalcopy 
\newtheorem{definition}{Definition}
\newtheorem{theorem}{Theorem}

\theoremstyle{remark}
\newtheorem{remark}{Remark}
\usepackage[capitalize,noabbrev]{cleveref}

\title{CoE: Collaborative Entropy for Uncertainty Quantification in Agentic Multi-LLM Systems}

\author{Kangkang Sun, Jun Wu, Jianhua Li, Minyi Guo, Xiuzhen Chen  \\
Shanghai Key Laboratory of Integrated Administration Technologies for Information Security \\
School of Computer Science\\
Shanghai Jiao Tong University\\
Shanghai, China \\
\texttt{\{szpsunkk, junwuhn, lijh888, myguo, chenxz\}@sjtu.edu.cn} \\
\And
Jianwei Huang \\
School of Science and Engineering \\
The Chinese University of Hong Kong, Shenzhen \\
Shenzhen, China \\
\texttt{\{jianweihuang\}@cuhk.edu.hk} \\
}

\iclrfinaltrue

\begin{document}

\maketitle

\begin{abstract}

Uncertainty estimation in multi-LLM systems remains largely single-model-centric: existing methods quantify uncertainty within each model but do not adequately capture semantic disagreement across models. To address this gap, we propose Collaborative Entropy (CoE), a unified information-theoretic metric for semantic uncertainty in multi-LLM collaboration. CoE is defined on a shared semantic cluster space and combines two components: intra-model semantic entropy and inter-model divergence to the ensemble mean. CoE is not a weighted ensemble predictor; it is a system-level uncertainty measure that characterizes collaborative confidence and disagreement.
We analyze several core properties of CoE, including non-negativity, zero-value certainty under perfect semantic consensus, and the behavior of CoE when individual models collapse to delta distributions. These results clarify when reducing per-model uncertainty is sufficient and when residual inter-model disagreement remains. We also present a simple CoE-guided, training-free post-hoc coordination heuristic as a practical application of the metric.
Experiments on \textit{TriviaQA} and \textit{SQuAD} with LLaMA-3.1-8B-Instruct, Qwen-2.5-7B-Instruct, and Mistral-7B-Instruct show that CoE provides stronger uncertainty estimation than standard entropy- and divergence-based baselines, with gains becoming larger as additional heterogeneous models are introduced. Overall, CoE offers a useful uncertainty-aware perspective on multi-LLM collaboration.
\end{abstract}

\section{Introduction}

Large language models (LLMs) have demonstrated remarkable capabilities across natural
language processing and decision-making tasks~\cite{chang2024evince, singhal2023large},
with systems such as GPT-4~\cite{achiam2023gpt}, Gemini~\cite{team2023gemini}, and
DeepSeek~\cite{liu2024deepseek} achieving strong performance on complex reasoning and
question-answering benchmarks. Yet these models remain susceptible to
\textit{hallucinations}~\cite{fadeeva2024fact}, \textit{solution space
bias}~\cite{holtzman2019curious}, and \textit{error
propagation}~\cite{prystawski2023think}—failure modes that are particularly consequential
in high-stakes domains such as medicine and law, where unreliable outputs can cause serious
harm.

Agentic multi-LLM frameworks have emerged as a principled response to these limitations.
By allowing heterogeneous models to cross-check reasoning through structured
interaction~\cite{knott2008multi, luo2025toward, li2025collaborative, ran2025mcce},
these systems leverage collective intelligence to improve both accuracy and reliability.
Empirical studies confirm that multi-LLM collaboration can reduce individual model
errors~\cite{wang2024rethinking}; however, they also reveal that collaboration quality
is highly uneven~\cite{cai2025survey, cemri2025multi}. A key obstacle is
\textit{semantic uncertainty}: when models disagree at the level of meaning—not merely
surface form—their outputs can conflict in ways that degrade rather than improve system
performance~\cite{abdelnabi2024cooperation, chan2023chateval, li2023theory}.
Understanding and measuring this disagreement is therefore a prerequisite for reliable
multi-LLM deployment.

\paragraph{The measurement gap.}
Uncertainty Quantification (UQ) is well-established for individual LLMs. Token-level
entropy~\cite{geng2024survey}, semantic entropy~\cite{farquhar2024detecting,
kuhn2023semantic}, self-consistency~\cite{rivera2024combining, ling2024uncertainty},
reflection~\cite{self2024self}, and conformal prediction~\cite{wang2024rethinking} all
provide principled estimates of a \emph{single} model's confidence. However, none of
these methods is designed to characterize uncertainty at the \emph{system} level in a
multi-LLM setting. The most natural extension—averaging per-model semantic entropy across
the ensemble—suffers from two fundamental limitations. First, it treats all models as
independent and interchangeable, ignoring the \textit{inter-model epistemic divergence}
that arises when heterogeneous models (e.g., LLaMA, Qwen, Mistral), trained on distinct
corpora and aligned via different procedures, disagree at the semantic level. Second, it
conflates two qualitatively different uncertainty sources—\textit{intra-model aleatoric
uncertainty} (a single model's internal ambiguity over semantic clusters) and
\textit{inter-model epistemic uncertainty} (cross-model semantic disagreement)—into a
single scalar, thereby discarding actionable diagnostic information about \emph{why} the
system is uncertain.

This conflation is not merely a theoretical inconvenience. In practice, the two
components call for different interventions: high intra-model uncertainty suggests that
individual models require more diverse sampling or better prompting, while high
inter-model uncertainty signals that the models are each internally confident but
semantically inconsistent with one another. A metric that collapses both into one number
forecloses this distinction entirely.

\paragraph{Collaborative Entropy.}
To fill this gap, we propose \textbf{Collaborative Entropy (CoE)}, a unified
information-theoretic metric for semantic uncertainty in agentic multi-LLM systems.
CoE is defined over a shared semantic cluster space and decomposes system-level
uncertainty into two interpretable components:
\begin{equation}
    \mathcal{U}_{\mathrm{CoE}}(\mathcal{K})
    \;=\;
    \underbrace{
        \frac{1}{|\mathcal{K}|}\sum_{i \in \mathcal{K}} SE(x_i)
    }_{\text{intra-model (aleatoric)}}
    \;+\;
    \underbrace{
        \sum_{i \in \mathcal{K}} w_i \cdot
        \mathcal{D}_{\mathrm{KL}}\!\left(p_i(\cdot|x) \,\|\, \bar{p}(\cdot|x)\right)
    }_{\text{inter-model (epistemic)}}
    \label{eq:coe_intro}
\end{equation}
where $SE(x_i)$ is the semantic entropy of model $\mathcal{M}_i$, $p_i(\cdot|x)$ is its
cluster-level output distribution, and $\bar{p}(\cdot|x) = \sum_i w_i p_i(\cdot|x)$ is
the weighted ensemble mean. Crucially, CoE is \emph{not} a weighted ensemble predictor
or an output-scoring rule: it is a system-level uncertainty measure that characterizes
collaborative confidence and cross-model semantic disagreement, independently of any
downstream task objective.

We analyze several core properties of CoE that formalize its behavior as a metric.
\textit{Non-Negativity} (Theorem~1) establishes that CoE is a valid uncertainty measure
bounded below by zero. \textit{Zero Value Certainty} (Theorem~2) proves that
$\mathcal{U}_{\mathrm{CoE}} = 0$ if and only if all models place full probability mass on
the same semantic cluster—i.e., perfect semantic consensus. \textit{Single-Model Negative
Entropy Maximization} (Theorem~3) characterizes the partial reduction achievable through
per-model uncertainty minimization: when each model collapses to a delta distribution, the
aleatoric component vanishes, but residual inter-model disagreement may persist, reflecting
irreducible epistemic divergence across heterogeneous models. Together, these results
clarify when reducing per-model uncertainty is sufficient for system-level certainty, and
when cross-model alignment is additionally required—a distinction invisible to
single-model-centric metrics. As a practical application of the metric, we also present a
simple CoE-guided, training-free post-hoc coordination heuristic that uses CoE-derived
signals to refine model weighting at inference time, without updating any model parameters.

\paragraph{Contributions.}
The main contributions of this paper are as follows:

\begin{itemize}

    \item \textit{A unified system-level uncertainty metric for multi-LLM collaboration.}
    We propose Collaborative Entropy (CoE), an information-theoretic metric that
    quantifies semantic uncertainty jointly across a heterogeneous set of LLMs. CoE is
    defined over a shared semantic cluster space and decomposes uncertainty into two
    interpretable components: intra-model semantic dispersion and inter-model semantic
    disagreement. It is the first metric to jointly capture both sources at the system
    level, and it is distinct from weighted ensemble predictors or output scoring rules.

    \item \textit{Formal theoretical analysis of CoE.}
    We prove three properties of CoE—Non-Negativity, Zero Value Certainty, and the
    behavior under per-model delta-distribution collapse—that collectively characterize
    CoE's range, its ideal state, and the conditions under which per-model uncertainty
    reduction does or does not suffice for system-level certainty. These results provide a
    principled foundation for interpreting CoE values in practice.

    \item \textit{A CoE-guided post-hoc coordination heuristic.}
    As a lightweight practical application of the metric, we present a training-free
    inference-time heuristic that uses CoE-related signals to adjust model weights and
    refine ensemble coordination. This procedure should be understood as a direct
    application of the metric rather than a standalone algorithmic contribution.

    \item \textit{Empirical validation as an uncertainty metric.}
    Experiments on TriviaQA and SQuAD with LLaMA-3.1-8B-Instruct,
    Qwen-2.5-7B-Instruct, and Mistral-7B-Instruct show that CoE provides stronger
    uncertainty estimation than standard entropy- and divergence-based baselines in AUROC
    and AURAC, with gains that scale as additional heterogeneous models are introduced.

\end{itemize}

The remainder of this paper is organized as follows.
Section~\ref{se: related work} reviews related work on uncertainty quantification for LLMs and multi-LLM systems.
Section~\ref{se: methodology} introduces the system model, the semantic entropy preliminaries on which CoE builds, and the formal definition of Collaborative Entropy together with its two-component decomposition.
Section~\ref{se: quadrants} interprets CoE through the lens of four qualitatively distinct uncertainty regimes on the $(\mathcal{U}_{\mathcal{A}},\,\mathcal{U}_{\mathcal{E}})$ plane.
Section~\ref{se: Experiment} presents the experimental setup and results.
Section~\ref{se: Conclusion} concludes the paper.
Appendix~\ref{se: Collaborative-Entropy-based Reduction optimization and Analysis} provides the formal theoretical analysis of CoE, including proofs of Non-Negativity, Zero Value Certainty, and Delta-Distribution Behaviour, and describes the CoE-guided post-hoc coordination heuristic.

\begin{figure*}[t]
    \centering
    \includegraphics[width=0.9\linewidth]{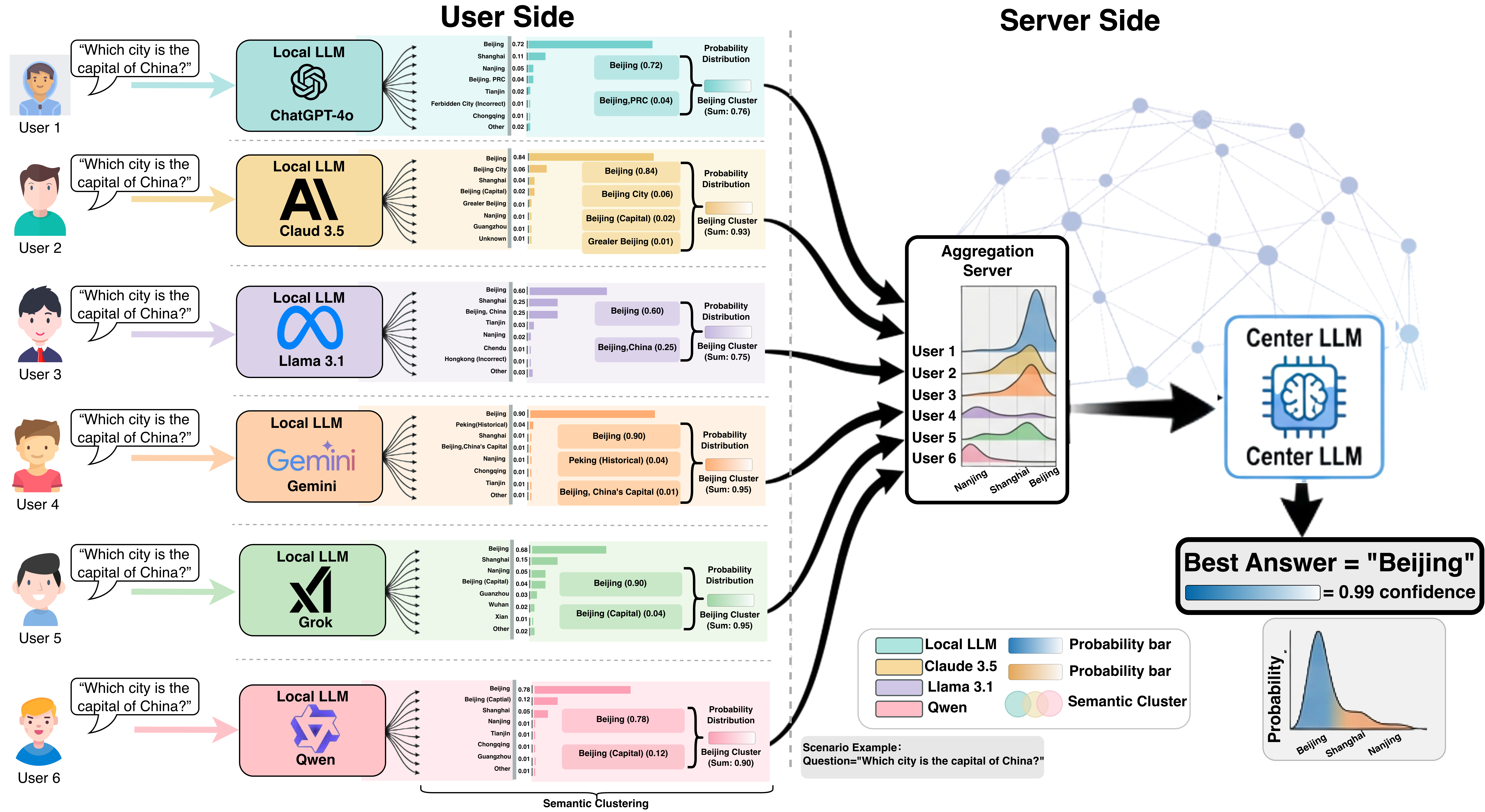}
    \caption{Overview of the proposed Collaborative Entropy (CoE) framework for agentic multi-LLM uncertainty quantification. Each LLM generates candidate responses that are clustered semantically to form probability distributions, which are then aggregated to minimize system-level collaborative entropy.}
    \label{fig: background}
\end{figure*}

\section{Related Work}
\label{se: related work}

UQ methods for LLMs span token-level entropy aggregation~\cite{geng2024survey},
self-verbalized confidence~\cite{rivera2024combining}, and semantic
entropy~\cite{farquhar2024detecting, kuhn2023semantic}—which clusters generations via
bidirectional entailment to provide paraphrase-invariant uncertainty estimates—as well
as classical approaches such as Bayesian networks and deep
ensembles~\cite{jospin2022hands} that scale poorly to modern LLMs. Benchmarks including
LM-Polygraph~\cite{fadeeva2023lm} and MAQA~\cite{yang2025maqa} reveal persistent
calibration gaps~\cite{ling2024uncertainty}, while recent surveys~\cite{shorinwa2025survey}
categorize open challenges along input ambiguity, reasoning divergence, and decoding
stochasticity. However, all these methods are \emph{single-model-centric}: they quantify
uncertainty within one model and cannot characterize semantic disagreement across a
heterogeneous ensemble. At the multi-LLM level, self-consistency~\cite{rivera2024combining}
treats cross-model agreement as a voting signal rather than a principled measure, and
graph-structured approaches such as DEUP~\cite{lahlou2021deup} and
TopologyUQ~\cite{da2025understanding} propagate uncertainty through ensemble pipelines
but do not define a closed-form metric over shared semantic clusters. The closest prior
work, \citet{kruse2025simple}, applies information-theoretic reasoning to multi-LLM
uncertainty but aggregates per-model entropy scores without decomposing aleatoric and
epistemic components or providing formal metric guarantees. CoE addresses this gap by
operating on a shared semantic cluster space—following the framework of semantic
entropy~\cite{farquhar2024detecting, kuhn2023semantic}—and introducing an asymmetric KL
divergence term measured against the ensemble mean, a reference structure that is
sensitive to directional epistemic divergence in ways that symmetric alternatives
(e.g., Jensen-Shannon) are not, as our experiments confirm.

\section{Methodology}
\label{se: methodology}

We begin by describing the multi-LLM collaborative system that motivates CoE
(Sec.~\ref{se: System Overview}), and then introduce the semantic entropy (SE)
preliminaries on which CoE is built (Sec.~\ref{se: Uncertainty Quantification Metrics
of agentic Multi-LLM Systems}).

\subsection{System Overview}
\label{se: System Overview}

Figure~\ref{fig: background} illustrates the proposed agentic multi-LLM collaborative
system, which consists of two primary layers.

\begin{itemize}

    \item \textbf{Local Inference and Clustering.}
    Each LLM agent $\mathcal{M}_i$ receives the same query $x$ and independently
    generates a set of candidate responses. These responses are partitioned into semantic
    clusters $\{c_j\}$ via bidirectional entailment, yielding a cluster-level probability
    distribution $p_i(c_j \mid x)$ that summarises the model's semantic confidence.

    \item \textbf{Central Aggregation.}
    A central aggregator collects the cluster distributions from all agents and computes
    the system-level Collaborative Entropy (CoE). The aggregator identifies the semantic
    cluster with the lowest associated CoE as the system's most reliable answer.

\end{itemize}

This two-layer structure cleanly separates the per-model uncertainty estimation
(intra-model) from the cross-model disagreement measurement (inter-model), corresponding
directly to the two components of CoE defined in Sec.~\ref{se: CoE formulation}.

\subsection{Semantic Entropy: Single-Model Uncertainty Baseline}
\label{se: Uncertainty Quantification Metrics of agentic Multi-LLM Systems}

Before defining CoE, we recall the semantic entropy (SE) metric of
\citet{farquhar2024detecting}, which serves as the intra-model component of CoE and as
the primary single-model baseline in our experiments.

\begin{definition}[Semantic Entropy, \citealp{farquhar2024detecting}]
\label{de: SE}
For an LLM $\mathcal{M}$ and input $x \in \mathcal{X}$, let $\{s^{(1)}, \dots, s^{(m)}\}$
be $m$ sampled output sequences, partitioned into $l$ semantic clusters
$\{c_1, \dots, c_l\} \subseteq \mathcal{C}$ via bidirectional entailment. The
\emph{semantic entropy} of $\mathcal{M}$ on $x$ is:
\begin{equation}
    SE(x) \;=\; -\sum_{j=1}^{l} p(c_j \mid x)\,\log\, p(c_j \mid x),
    \label{eq: SE}
\end{equation}
where $p(c_j \mid x) = \sum_{s^{(k)} \in c_j} p(s^{(k)} \mid x)$ aggregates the
length-normalized generation probabilities of all sequences assigned to cluster $c_j$.
\end{definition}

SE measures how spread the model's probability mass is across semantically distinct
answers: $SE(x) = 0$ when the model concentrates entirely on one semantic cluster, and
$SE(x)$ is maximised when probability mass is distributed uniformly. SE is robust to
surface paraphrase because it operates on semantic clusters rather than raw token
sequences. However, as a single-model metric, SE cannot capture the disagreement that
arises when \emph{multiple} heterogeneous models each have low internal entropy yet
commit to \emph{different} semantic clusters. Addressing this limitation is the central
motivation for CoE.

\subsection{Collaborative Entropy: Decomposing System-Level Uncertainty}
\label{se: CoE formulation}

We now define CoE formally. Let $\mathcal{K}$ denote a set of $K$ LLMs
$\mathcal{M} = \{\mathcal{M}_1, \dots, \mathcal{M}_K\}$ and let
$\{w_1, \dots, w_K\}$ be non-negative ensemble weights satisfying
$\sum_{i=1}^{K} w_i = 1$. For input $x$, each model $\mathcal{M}_i$ generates $m_i$
sequences that are collectively partitioned into $l$ shared semantic clusters
$\{c_1, \dots, c_l\} \subseteq \mathcal{C}$ via bidirectional entailment.
Let $p_i(\cdot \mid x)$ denote the resulting cluster distribution of model
$\mathcal{M}_i$, and let $\bar{p}(\cdot \mid x) = \sum_{i=1}^{K} w_i\, p_i(\cdot \mid x)$
denote the weighted ensemble mean distribution.

CoE decomposes system-level semantic uncertainty into two interpretable components:

\paragraph{Intra-model aleatoric uncertainty $\mathcal{U}_{\mathcal{A}}$.}
This term captures the average internal ambiguity of individual models:
\begin{equation}
\label{eq: U_A}
\mathcal{U}_{\mathcal{A}}(\mathcal{K})
\;=\; \frac{1}{K} \sum_{i=1}^{K} SE(x_i).
\end{equation}
High $\mathcal{U}_{\mathcal{A}}$ indicates that individual models are themselves uncertain,
regardless of whether they agree with each other.

\paragraph{Inter-model epistemic uncertainty $\mathcal{U}_{\mathcal{E}}$.}
This term captures the semantic divergence of each model's distribution from the ensemble
consensus:
\begin{equation}
\label{eq: U_E}
\mathcal{U}_{\mathcal{E}}(\mathcal{K})
\;=\; \sum_{i=1}^{K} w_i \cdot
\mathcal{D}\!\left(p_i(\cdot \mid x) \;\|\; \bar{p}(\cdot \mid x)\right),
\end{equation}
where $\mathcal{D}(\cdot \| \cdot)$ is a distributional divergence. In our primary
formulation we use the asymmetric KL divergence,
$\mathcal{D}_{\mathrm{KL}}(p_i \| \bar{p})
= \sum_{j=1}^{l} p_i(c_j \mid x)\log\frac{p_i(c_j \mid x)}{\bar{p}(c_j \mid x)}$,
whose reference-directed asymmetry is critical for capturing directional epistemic
disagreement (see Sec.~\ref{se: Experiment}). We also evaluate symmetric alternatives
(JS, Wasserstein, Hellinger) as ablations. High $\mathcal{U}_{\mathcal{E}}$ indicates
that models disagree semantically even if each is individually confident.

\paragraph{Collaborative Entropy (CoE).}

\begin{definition}[Collaborative Entropy]
\label{de: CoE}
The \emph{Collaborative Entropy} of a multi-LLM system $\mathcal{K}$ on input $x$ is:
\begin{equation}
\label{eq: objective}
\mathcal{U}_{CoE}(\mathcal{K})
\;=\;
\underbrace{
    \frac{1}{K}\sum_{i=1}^{K} SE(x_i)
}_{\mathcal{U}_{\mathcal{A}}:\;\text{intra-model}}
\;+\;
\underbrace{
    \sum_{i=1}^{K} w_i \cdot
    \mathcal{D}_{\mathrm{KL}}\!\left(p_i(\cdot \mid x) \;\|\; \bar{p}(\cdot \mid x)\right)
}_{\mathcal{U}_{\mathcal{E}}:\;\text{inter-model}}.
\end{equation}
\end{definition}

\noindent
CoE is a \emph{measurement} of system-level semantic uncertainty: it is not a
task-specific scoring rule, a weighted ensemble predictor, or an output selector. Its
two components are interpretable and actionable—high $\mathcal{U}_{\mathcal{A}}$ calls
for improving per-model sampling or prompting, while high $\mathcal{U}_{\mathcal{E}}$
calls for inter-model alignment. The theoretical properties of CoE are analysed in
Appendix~\ref{se: Collaborative-Entropy-based Reduction optimization and Analysis}.

\section{CoE Decomposition: The Four Uncertainty Quadrants}
\label{se: quadrants}

\begin{figure}[h]
    \centering
    \includegraphics[width=0.52\textwidth]{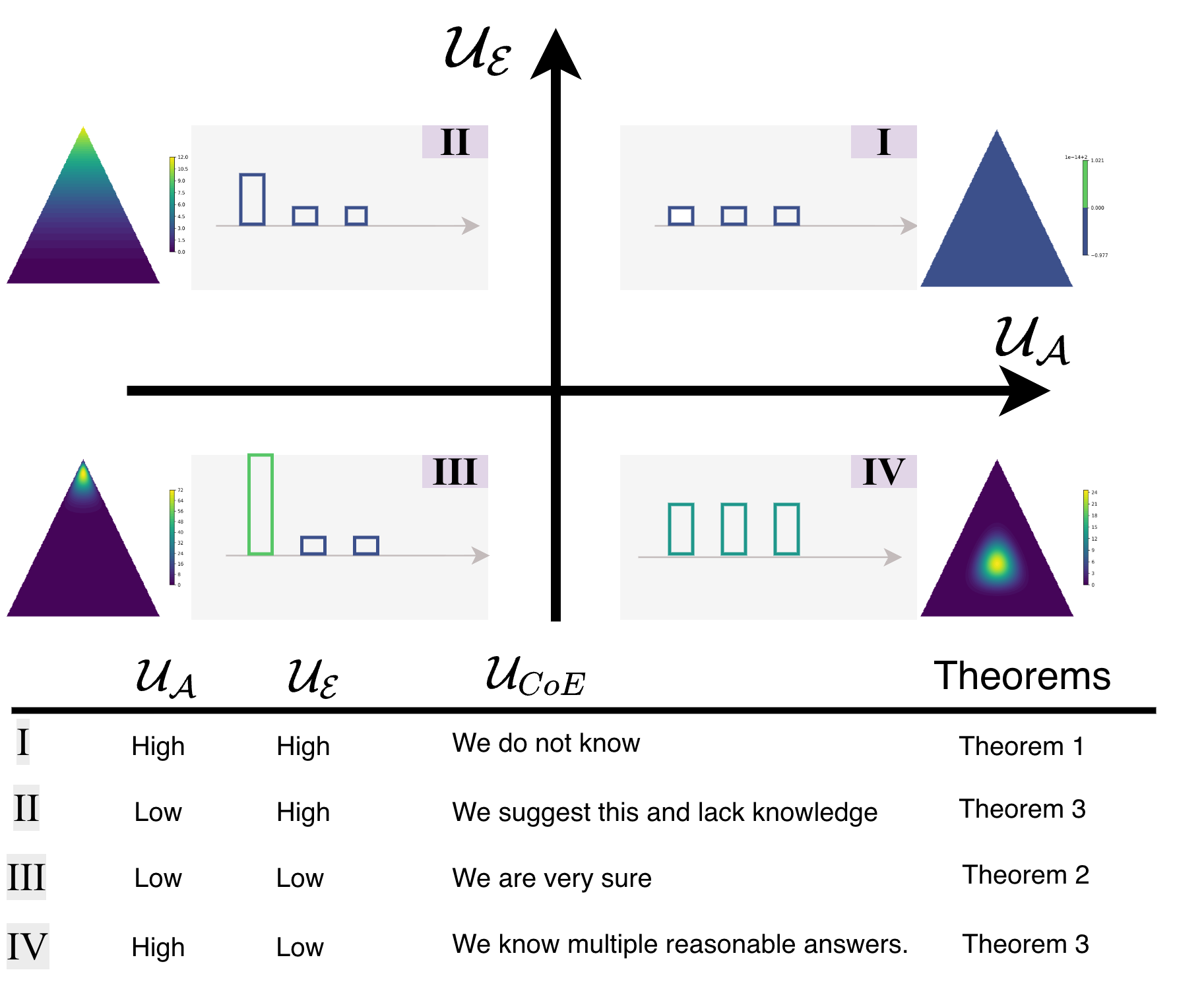}
    \caption{
        Visualization of the four CoE quadrants using Dirichlet distributions
        over three semantic clusters. Each simplex (colour intensity = probability
        density) and accompanying bar chart illustrates a characteristic combination
        of $\mathcal{U}_{\mathcal{A}}$ (horizontal axis) and
        $\mathcal{U}_{\mathcal{E}}$ (vertical axis).
    }
    \label{fig: eu-au}
\end{figure}

Figure~\ref{fig: eu-au} maps the four qualitatively distinct states of a multi-LLM
system onto the $(\mathcal{U}_{\mathcal{A}},\,\mathcal{U}_{\mathcal{E}})$ plane.

\begin{enumerate}

    \item \textbf{High $\mathcal{U}_{\mathcal{A}}$ + High $\mathcal{U}_{\mathcal{E}}$
    — ``We do not know.''}
    Both intra-model ambiguity and inter-model disagreement are severe.
    CoE is maximally accumulated; Theorem~\ref{th: non-negativity} establishes
    that $\mathcal{U}_{CoE}\geq 0$ holds as a global lower bound throughout
    this worst-case regime.

    \item \textbf{Low $\mathcal{U}_{\mathcal{A}}$ + High $\mathcal{U}_{\mathcal{E}}$
    — ``Models are confident but disagree.''}
    Individual models have concentrated distributions (low internal entropy),
    yet they commit to \emph{different} semantic clusters.
    This is precisely Case B of Theorem~\ref{th: single-LLM negative entropy maximization}:
    per-model entropy maximisation has eliminated $\mathcal{U}_{\mathcal{A}}$,
    but residual epistemic divergence $\mathcal{U}_{\mathcal{E}}>0$ persists
    and can only be reduced through inter-model alignment.

    \item \textbf{Low $\mathcal{U}_{\mathcal{A}}$ + Low $\mathcal{U}_{\mathcal{E}}$
    — ``We are very sure.''}
    All models concentrate on the same cluster.
    By Theorem~\ref{th: zero value certainty}, $\mathcal{U}_{CoE}=0$; this is the
    ideal state and the target of the CoE-guided coordination heuristic.

    \item \textbf{High $\mathcal{U}_{\mathcal{A}}$ + Low $\mathcal{U}_{\mathcal{E}}$
    — ``Multiple reasonable answers exist.''}
    Models broadly agree on the distribution over clusters but none is internally
    decisive.
    Per Theorem~\ref{th: single-LLM negative entropy maximization}(i), driving
    each model toward its Delta distribution reduces $\mathcal{U}_{\mathcal{A}}$
    and, if consensus is reached, eliminates $\mathcal{U}_{CoE}$ entirely.

\end{enumerate}

\section{Experiments}
\label{se: Experiment}

\subsection{Experimental Setup}
\label{se: setup}

\paragraph{Models.}
We evaluate on three publicly available instruction-tuned LLMs:
LLaMA-3.1-8B-Instruct~\citep{dubey2024llama},
Qwen-2.5-7B-Instruct~\citep{team2024qwen2}, and
Mistral-7B-Instruct~\citep{jiang2024mixtral},
abbreviated as \textbf{Llama}, \textbf{Qwen}, and \textbf{Mistral} throughout.
All models are used without any fine-tuning or parameter modification.
Experiments are conducted on the 7--8B parameter range to ensure reproducibility
and to isolate the effect of multi-LLM collaboration from model-scale confounds;
we discuss the implications of this choice in Sec.~\ref{se: Conclusion}.

\paragraph{Datasets.}
We evaluate on two open-domain QA benchmarks.
\textit{TriviaQA}~\citep{joshi2017triviaqa} is a large-scale reading comprehension
dataset requiring retrieval and compositional reasoning over noisy evidence documents.
\textit{SQuAD}~\citep{rajpurkar2016squad} is an extractive QA dataset built on
Wikipedia articles that demands precise context grounding.
The two benchmarks present complementary challenges: TriviaQA favours broad factual
recall, while SQuAD rewards fine-grained span identification.
Both evaluations use 200 samples drawn from the respective validation splits.
Detailed experimental results are provided in Appendix~\ref{sec: Experimental Results}.

\paragraph{Evaluation metrics.}
Following \citet{farquhar2024detecting}, we evaluate uncertainty metrics along three
complementary axes. \textbf{AUROC}~\citep{mcdermott2024closer}: the area under the
    receiver-operating-characteristic curve, measuring how well the uncertainty score
    ranks correct answers above incorrect ones across all thresholds.
    A value of 1.0 indicates perfect discrimination; 0.5 corresponds to random ranking. \textbf{AURAC}~\citep{kuhn2023semantic}: the area under the
    rejection-accuracy curve, summarising how accuracy improves as the most uncertain
    predictions are progressively abstained from.
    Higher values indicate better selective prediction across all coverage levels. \textbf{Rejection Accuracy}~\citep{nadeem2009accuracy}: accuracy on the
    retained predictions at fixed retention levels (80\%, 90\%, 95\%, 100\%), capturing
    the practical precision--coverage trade-off at specific operating points.

\noindent
AUROC and AURAC assess CoE as a \emph{ranking} signal for uncertainty, while
rejection accuracy assesses it as a \emph{threshold} signal.
Taken together, they provide a comprehensive picture of whether CoE reliably
identifies which predictions to trust.

\subsubsection{Baseline Methods}
\label{se: baselines}

\paragraph{Single-model uncertainty baselines.}
We compare CoE against three standard single-model UQ methods.
For multi-model settings, each baseline is reported as the arithmetic mean over all
constituent models, i.e., $\frac{1}{K}\sum_i \text{metric}(\mathcal{M}_i, x)$. \textbf{Semantic Entropy (SE)}~\citep{farquhar2024detecting}:
    partitions sampled responses into semantic clusters via bidirectional entailment
    and computes Shannon entropy over the cluster distribution.
    When averaged uniformly over $K$ models, this equals the aleatoric component
    $\mathcal{U}_{\mathcal{A}}$ of CoE (Eq.~\ref{eq: U_A}).
    The two are therefore numerically equivalent under uniform weighting but differ
    in interpretation: $\overline{SE}$ treats models as independent, while
    $\mathcal{U}_{\mathcal{A}}$ is framed explicitly as one component of a
    joint decomposition. \textbf{Regular Entropy}~\citep{farquhar2024detecting}:
    computes token-level entropy over the raw generation distribution, without
    semantic clustering. This metric is known to be misleading when semantically
    equivalent answers receive different token sequences. \textbf{$P_{\text{false}}$}~\citep{madaan2023self}:
    the probability assigned by the LLM to the response ``False'' in a
    self-consistency judgment, used as a proxy for model-level uncertainty. We additionally include $\mathcal{U}_{\mathcal{A}}$ and $\mathcal{U}_{\mathcal{E}}$
as \emph{ablation baselines} to isolate the individual contribution of each CoE
component.

\paragraph{Divergence measure ablations.}
To assess the sensitivity of CoE to the choice of inter-model divergence
$\mathcal{D}(\cdot \| \cdot)$ in Eq.~(\ref{eq: objective}), we evaluate four
instantiations that span a range of structural properties. \textbf{KL divergence} $\mathcal{D}_{\mathrm{KL}}$:
    asymmetric and unbounded; measures directed divergence from each model's
    distribution to the ensemble mean. \textbf{Jensen-Shannon divergence} $\mathcal{D}_{\mathrm{JS}}$~\citep{kruse2025simple}:
    symmetric and bounded in $[0, \log 2]$; the symmetrised average of two
    KL terms. \textbf{Wasserstein distance} $\mathcal{D}_{W}$~\citep{arjovsky2017wasserstein}:
    geometry-aware optimal-transport distance between distributions over the
    cluster space. \textbf{Hellinger distance} $\mathcal{D}_{H}$~\citep{beran1977minimum}:
    symmetric, bounded in $[0, 1]$, based on the $\ell_2$ distance between
    square-root densities.
\noindent
These four divergences are asymmetric vs.\ symmetric, bounded vs.\ unbounded, and
geometry-aware vs.\ geometry-agnostic, enabling a controlled comparison of which
structural property matters most for CoE.

\subsection{CoE as an Uncertainty Metric}
\label{se: metric eval}

\begin{table*}[t]
\centering
\caption{
    Selective prediction performance of uncertainty metrics in multi-LLM
    collaborative settings on \textit{TriviaQA} (200 samples).
    For Semantic Entropy, Regular Entropy, and $P_{\text{false}}$, scores are
    reported as the arithmetic mean over constituent models.
    $\mathcal{U}_{\mathcal{A}}$ and $\mathcal{U}_{\mathcal{E}}$ are ablation
    baselines isolating each component of CoE.
    Bold denotes the best AUROC in each configuration.
}
\label{tab:2}
\resizebox{\linewidth}{!}{%
\begin{tabular}{lcccccccccccccc}
\toprule
\multirow{3}{*}{Uncertainty Metric}
    & \multicolumn{6}{c}{Two LLMs (Llama + Qwen)}
    & \phantom{x}
    & \multicolumn{6}{c}{Three LLMs (Llama + Qwen + Mistral)} \\
\cmidrule(lr){2-7}\cmidrule(lr){9-14}
    & \multicolumn{4}{c}{Rejection Accuracy}
    & \multirow{2}{*}{AURAC}
    & \multirow{2}{*}{AUROC}
    &
    & \multicolumn{4}{c}{Rejection Accuracy}
    & \multirow{2}{*}{AURAC}
    & \multirow{2}{*}{AUROC} \\
\cmidrule(lr){2-5}\cmidrule(lr){9-12}
    & 80\% & 90\% & 95\% & 100\%
    &       &
    &
    & 80\% & 90\% & 95\% & 100\%
    &       & \\
\midrule
Semantic Entropy
    & 0.625 & 0.583 & 0.579 & 0.575 & 0.176 & 0.670
    &
    & 0.394 & 0.416 & 0.468 & 0.375 & 0.123 & 0.687 \\
Regular Entropy
    & 0.562 & 0.583 & 0.605 & 0.575 & 0.175 & 0.475
    &
    & 0.342 & 0.333 & 0.343 & 0.375 & 0.103 & 0.416 \\
$P_{\text{false}}$
    & 0.575 & 0.575 & 0.575 & 0.575 & 0.172 & 0.586
    &
    & 0.375 & 0.405 & 0.405 & 0.375 & 0.117 & 0.540 \\
\rowcolor[rgb]{0.851,0.953,0.992}
$\mathcal{U}_{\mathcal{E}}$
    & 0.625 & 0.583 & 0.589 & 0.575 & 0.177 & 0.661
    &
    & 0.368 & 0.388 & 0.437 & 0.375 & 0.116 & 0.716 \\
\rowcolor[rgb]{0.851,1,0.851}
$\mathcal{U}_{\mathcal{A}}$
    & 0.593 & 0.583 & 0.578 & 0.575 & 0.174 & 0.666
    &
    & 0.394 & 0.416 & 0.437 & 0.375 & 0.121 & 0.673 \\
\rowcolor[rgb]{0.925,0.851,0.925}
$\mathcal{U}_{CoE}$
    & 0.593 & 0.583 & 0.578 & 0.575 & 0.174 & \textbf{0.683}$\uparrow$
    &
    & 0.394 & 0.416 & 0.437 & 0.375 & 0.121 & \textbf{0.772}$\uparrow$ \\
\bottomrule
\end{tabular}}
\end{table*}

\paragraph{Results on TriviaQA (Table~\ref{tab:2}).}
We compare six uncertainty metrics across two ensemble configurations on TriviaQA
(200 samples): two-model (Llama + Qwen) and three-model (Llama + Qwen + Mistral). In the two-model setting, CoE achieves the highest AUROC of \textbf{0.683}, surpassing
Semantic Entropy (0.670), $\mathcal{U}_{\mathcal{A}}$ (0.666),
$\mathcal{U}_{\mathcal{E}}$ (0.661), $P_{\text{false}}$ (0.586), and
Regular Entropy (0.475).
Rejection accuracy remains consistently above 0.58 across all retention levels,
indicating well-calibrated selective prediction even with limited inter-model
diversity in a two-model ensemble.
Notably, CoE outperforms both of its individual components
($\mathcal{U}_{\mathcal{A}}$ and $\mathcal{U}_{\mathcal{E}}$) taken alone,
confirming that the two uncertainty sources carry complementary information that is
better captured jointly.

As the ensemble scales to three models, CoE's advantage widens.
It attains an AUROC of \textbf{0.772}, a substantial improvement over the
strongest single-component baseline $\mathcal{U}_{\mathcal{E}}$ (0.716) and over
Semantic Entropy (0.687).
This \emph{scaling behaviour}—where CoE's margin over baselines grows with ensemble
size—is consistent with Theorem~\ref{th: single-LLM negative entropy maximization}:
as more heterogeneous models contribute, inter-model epistemic divergence becomes an
increasingly significant contributor to system uncertainty, and CoE's explicit
modelling of this component yields proportionally larger gains.

\section{Conclusion}
\label{se: Conclusion}

We have introduced \textbf{Collaborative Entropy (CoE)}, an information-theoretic
metric for quantifying semantic uncertainty in agentic multi-LLM systems.
CoE decomposes system-level uncertainty into two interpretable components—intra-model
aleatoric uncertainty, captured by the mean semantic entropy across models, and
inter-model epistemic uncertainty, captured by the weighted asymmetric KL divergence
from each model's cluster distribution to the ensemble mean—and provides three formal
guarantees: Non-Negativity, Zero Value Certainty, and a characterisation of the partial
reductions achievable through per-model entropy maximisation.
These properties clarify precisely when individual-model certainty is sufficient for
system-level certainty, and when residual inter-model disagreement persists as
irreducible epistemic divergence.

Experiments on TriviaQA and SQuAD with three 7--8B instruction-tuned LLMs demonstrate
that CoE consistently outperforms single-model uncertainty baselines (semantic entropy,
regular entropy, $P_{\text{false}}$) and single-component ablations
($\mathcal{U}_{\mathcal{A}}$, $\mathcal{U}_{\mathcal{E}}$) in AUROC and AURAC, with
gains that scale as additional heterogeneous models are introduced (AUROC
$0.683 \to 0.772$ on TriviaQA; up to $0.878$ on SQuAD).
As a practical application of the metric, the CoE-guided coordination heuristic
achieves a $+39.0\%$ accuracy gain under asymmetric KL divergence, substantially
exceeding symmetric alternatives and validating the importance of reference-directed
epistemic measurement.

\paragraph{Limitations.}
We acknowledge two theoretical limitations inherent to the aleatoric/epistemic
decomposition adopted in CoE.
First, intra-model semantic entropy may conflate aleatoric and epistemic sources at
the single-model level: an LLM that lacks knowledge about a prompt may hallucinate
diverse but factually incorrect answers across samples, artificially inflating SE
beyond pure aleatoric uncertainty~\citep{shorinwa2025survey}.
This \emph{epistemic contamination of SE} is a known open problem in LLM uncertainty
quantification and is orthogonal to the multi-LLM contribution of CoE.
Second, inter-model KL divergence may capture superficial stylistic differences
between models—such as verbosity or formatting preferences induced by differing RLHF
procedures—rather than genuine epistemic disagreement about the correct answer.
Future work could incorporate a semantic normalisation step (e.g., extracting canonical
answer spans before clustering) to reduce the influence of stylistic variation on
$\mathcal{U}_{\mathcal{E}}$.

\paragraph{Future directions.}
Several extensions of CoE are natural.
Scaling to larger models (e.g., 70B parameter LLMs) and to closed-source APIs—where
token-level probabilities may be unavailable and generation frequency must serve as a
proxy for cluster probability—would broaden the applicability of the metric.
Incorporating verbalized confidence~\citep{rivera2024combining} or
log-probability-based signals as additional uncertainty sources could further enrich
the decomposition.
Finally, applying CoE to sequential multi-agent settings, where models interact across
multiple rounds, would test whether the metric's two-component structure remains
informative when uncertainty accumulates over a reasoning trajectory.

\section{Acknowledgment}
This work is supported by the National Natural Science Foundation of China (Project 62271434, 62501397), Shenzhen Key Laboratory of Crowd Intelligence Empowered Low-Carbon Energy Networks (No. ZDSYS20220606100601002), the Shenzhen Stability Science Program 2023, Shenzhen Loop Area Institute, and Shenzhen Institute of Artificial Intelligence and Robotics for Society.

\bibliography{iclr2026_conference}
\bibliographystyle{iclr2026_conference}

\newpage
\appendix
\onecolumn

\section{Theoretical Analysis of CoE and a CoE-Guided Coordination Heuristic}
\label{se: Collaborative-Entropy-based Reduction optimization and Analysis}

We first establish three formal properties of CoE as an uncertainty metric
(Sec.~\ref{se: CoE analysis}), then present a lightweight coordination heuristic
motivated by these properties (Sec.~\ref{se: CoE-Reduction Optimization Algorithm}). 

\subsection{Theoretical Properties of CoE}
\label{se: CoE analysis}

\begin{theorem}[Non-Negativity]
\label{th: non-negativity}
For any input $x$ and model set
$\mathcal{M} = \{\mathcal{M}_1, \dots, \mathcal{M}_K\}$,
$\mathcal{U}_{CoE}(\mathcal{K}) \geq 0$.
\end{theorem}

\begin{proof}[Proof of Theorem~\ref{th: non-negativity} (Non-Negativity)]
The aleatoric term $\mathcal{U}_{\mathcal{A}} = \frac{1}{K}\sum_{i}SE_i(x) \geq 0$
follows immediately from the non-negativity of Shannon entropy: for any distribution
over $l$ clusters, $-\sum_j p(c_j)\log p(c_j) \geq 0$, since
$p(c_j)\in[0,1]$ and $u\log u \leq 0$ for $u\in(0,1]$.

The epistemic term satisfies
\begin{equation}
    \sum_{i=1}^{K} w_i\,\mathcal{D}_{\mathrm{KL}}(p_i \| \bar{p})
    \;=\;
    \sum_{i=1}^{K} w_i \sum_{j=1}^{l}
    p_i(c_j\mid x)\log\frac{p_i(c_j\mid x)}{\bar{p}(c_j\mid x)}
    \;\geq\; 0,
    \label{eq: proof1}
\end{equation}
by the non-negativity of KL divergence, with equality iff
$p_i(\cdot\mid x)=\bar{p}(\cdot\mid x)$ for all $i$.
Since $w_i\geq 0$ and $\sum_i w_i=1$, the weighted sum is also non-negative.
As $\mathcal{U}_{CoE}$ is the sum of two non-negative terms,
$\mathcal{U}_{CoE}(\mathcal{K})\geq 0$.
\end{proof}

\begin{theorem}[Zero Value Certainty]
\label{th: zero value certainty}
$\mathcal{U}_{CoE}(\mathcal{K}) = 0$ if and only if there exists a single semantic
cluster $c^*$ such that $p_i(c^* \mid x) = 1$ for all $i \in \{1, \dots, K\}$,
i.e., every model places all its probability mass on the same semantic cluster.
\end{theorem}

\begin{proof}[Proof of Theorem~\ref{th: zero value certainty} (Zero Value Certainty)]
Suppose $p_i(c\mid x)=1$ for all $i$ and some fixed cluster $c$.
The ensemble mean then satisfies $\bar{p}(c\mid x)=\sum_i w_i\cdot 1=1$
and $\bar{p}(c_j\mid x)=0$ for all $j\neq c$.
The aleatoric term reduces to
\[
    -\sum_{j=1}^{l}\bar{p}(c_j\mid x)\log\bar{p}(c_j\mid x)
    \;=\; -1\cdot\log 1 \;-\; \sum_{j\neq c} 0
    \;=\; 0.
\]
For the epistemic term, since $p_i(c\mid x)=\bar{p}(c\mid x)=1$ and
$p_i(c_j\mid x)=\bar{p}(c_j\mid x)=0$ for $j\neq c$,
\[
    \mathcal{D}_{\mathrm{KL}}(p_i\|\bar{p})
    = 1\cdot\log\tfrac{1}{1} + \sum_{j\neq c} 0
    = 0 \quad \forall\,i,
\]
so $\sum_i w_i\,\mathcal{D}_{\mathrm{KL}}(p_i\|\bar{p})=0$.
Hence $\mathcal{U}_{CoE}(\mathcal{K})=0$.

Conversely, if $\mathcal{U}_{CoE}=0$, then both terms vanish.
$\mathcal{U}_{\mathcal{A}}=0$ requires each $SE_i(x)=0$,
i.e., every $p_i$ is a Delta distribution.
$\mathcal{U}_{\mathcal{E}}=0$ then requires
$\mathcal{D}_{\mathrm{KL}}(p_i\|\bar{p})=0$ for all $i$,
i.e., $p_i=\bar{p}$ for all $i$, which holds iff all Delta
distributions concentrate on the same cluster.
\end{proof}

\begin{remark}
Theorem~\ref{th: zero value certainty} makes precise the ideal state of a multi-LLM
system: CoE reaches its global minimum of zero only when all models are simultaneously
internally certain \emph{and} mutually consistent. Neither condition alone is sufficient.
\end{remark}

To characterise the partial reductions of CoE achievable through per-model optimisation,
we introduce the following:

\begin{definition}[Delta Distribution]
\label{def: delta}
The cluster distribution $p_i(\cdot \mid x)$ of model $\mathcal{M}_i$ is called a
\emph{Delta distribution}, written $p_i^*(\cdot \mid x)$, if there exists a unique
cluster $c_i^*$ such that $p_i(c_i^* \mid x) = 1$ and $p_i(c_j \mid x) = 0$ for all
$j \neq c_i^*$.
\end{definition}

\begin{remark}
$p_i^*$ is the unique global maximiser of $-SE_i(x)$ over the probability simplex
$\Delta^{l-1}$, since $SE_i(x) \geq 0$ with equality if and only if $p_i$ is a Delta
distribution~\citep{shannon1948mathematical}. Consequently, driving model $\mathcal{M}_i$
toward $p_i^*$ eliminates its contribution to $\mathcal{U}_{\mathcal{A}}$ entirely, but
does not by itself control $\mathcal{U}_{\mathcal{E}}$.
\end{remark}

\begin{theorem}[Delta-Distribution Behaviour]
\label{th: single-LLM negative entropy maximization}
Let $\mathcal{M} = \{\mathcal{M}_1, \dots, \mathcal{M}_K\}$ with fixed non-negative
weights $\{w_i\}$ satisfying $\sum_{i=1}^{K} w_i = 1$, and let each
$p_i^*(c_j \mid x)$ be the Delta distribution of $\mathcal{M}_i$ concentrating on
cluster $c_{j_i^*}$. Suppose $p_i = p_i^*$ for all $i$. Then:
\begin{enumerate}
    \item \textbf{Aleatoric minimum.}
    $\mathcal{U}_{\mathcal{A}} = 0$.

    \item \textbf{Case A — Consensus.}
    If $j_1^* = j_2^* = \cdots = j_K^* =: j^*$, then
    $\mathcal{U}_{\mathcal{E}} = 0$ and $\mathcal{U}_{CoE}(\mathcal{K}) = 0$,
    the global minimum of Theorem~\ref{th: zero value certainty}.

    \item \textbf{Case B — Residual Disagreement.}
    If $\exists\, i \neq k$ such that $j_i^* \neq j_k^*$, then
    $\mathcal{U}_{\mathcal{E}} > 0$ and
    \begin{equation}
        \mathcal{U}_{CoE}(\mathcal{K})
        \;=\; \mathcal{U}_{\mathcal{E}}
        \;=\; \sum_{i=1}^{K} w_i \cdot
        \mathcal{D}_{\mathrm{KL}}\!\left(p_i^* \;\|\; \bar{p}^*\right),
        \label{eq: caseB}
    \end{equation}
    where $\bar{p}^*(c_j \mid x) = \sum_i w_i\,\mathbf{1}[j = j_i^*]$.
\end{enumerate}
\end{theorem}

\begin{proof}[Proof of Theorem~\ref{th: single-LLM negative entropy maximization}
(Delta-Distribution Behaviour)]

\noindent\textbf{Part (i): $\mathcal{U}_{\mathcal{A}}=0$.}
Shannon entropy satisfies $SE_i(x)\geq 0$ for any distribution over $l$ clusters,
with equality iff $p_i$ is a Delta distribution (Definition~\ref{def: delta}).
When every $p_i=p_i^*$ (a Delta distribution),
$SE_i(x)=0$ for all $i$, so $\mathcal{U}_{\mathcal{A}}
=\frac{1}{K}\sum_i SE_i(x)=0$.

\noindent\textbf{Part (ii): Case A (Consensus).}
Suppose $j_1^*=\cdots=j_K^*=j^*$, i.e., all Delta distributions concentrate on $c_{j^*}$.
The ensemble mean satisfies $\bar{p}^*(c_{j^*}\mid x)=\sum_i w_i\cdot 1=1$
and $\bar{p}^*(c_j\mid x)=0$ for $j\neq j^*$. For each model,
\[
    \mathcal{D}_{\mathrm{KL}}(p_i^*\|\bar{p}^*)
    = 1\cdot\log\frac{1}{1} = 0,
\]
so $\mathcal{U}_{\mathcal{E}}=\sum_i w_i\cdot 0=0$ and
$\mathcal{U}_{CoE}=0$, consistent with Theorem~\ref{th: zero value certainty}.

\noindent\textbf{Part (iii): Case B (Residual Disagreement).}
Suppose $\exists\,i\neq k$ with $j_i^*\neq j_k^*$.
Without loss of generality, let $p_1^*(c_1^*\mid x)=1$ and
$p_2^*(c_2^*\mid x)=1$ with $c_1^*\neq c_2^*$.
The ensemble mean satisfies $\bar{p}^*(c_1^*\mid x)=w_1>0$
and $\bar{p}^*(c_2^*\mid x)=w_2>0$.
For model $\mathcal{M}_1$,
\[
    \mathcal{D}_{\mathrm{KL}}(p_1^*\|\bar{p}^*)
    = 1\cdot\log\frac{1}{w_1} = -\log w_1 > 0,
\]
since $w_1<1$. Hence $\mathcal{U}_{\mathcal{E}}>0$, and
\[
    \mathcal{U}_{CoE}(\mathcal{K})
    = \underbrace{\mathcal{U}_{\mathcal{A}}}_{=\,0}
    + \mathcal{U}_{\mathcal{E}}
    = \sum_{i=1}^{K} w_i\,\mathcal{D}_{\mathrm{KL}}(p_i^*\|\bar{p}^*) > 0,
\]
where $\bar{p}^*(c_j\mid x)=\sum_i w_i\,\mathbf{1}[j=j_i^*]$.
This residual CoE is irreducible by per-model optimisation alone and
reflects genuine epistemic disagreement across heterogeneous models,
corresponding to Quadrant II in Figure~\ref{fig: eu-au}.
\end{proof}

\begin{remark}
Theorem~\ref{th: single-LLM negative entropy maximization} reveals a fundamental
asymmetry in how the two CoE components can be reduced. The aleatoric component
$\mathcal{U}_{\mathcal{A}}$ can always be driven to zero by per-model
negative-entropy maximisation, independently of what other models do. The epistemic
component $\mathcal{U}_{\mathcal{E}}$, however, cannot be controlled by any single model:
it vanishes only when \emph{all} models converge to the same cluster (Case A), and
persists otherwise as irreducible inter-model disagreement (Case B). This distinction
is invisible to single-model UQ metrics and motivates the decomposed structure of CoE.
Theorem~\ref{th: single-LLM negative entropy maximization} is stated over the semantic
cluster distribution induced by sampling and does not imply deterministic token-level
generation.
\end{remark}

\subsection{A CoE-Guided Post-Hoc Coordination Heuristic}
\label{se: CoE-Reduction Optimization Algorithm}

Theorems~\ref{th: non-negativity}--\ref{th: single-LLM negative entropy maximization}
characterise when and how CoE can be reduced. We now describe a simple, training-free,
post-hoc heuristic that operationalises this intuition at inference time. This procedure
is best understood as a practical application of the CoE metric rather than a standalone
algorithmic contribution: it uses CoE-derived signals to reweight the ensemble and refine
model coordination, without updating any model parameters. Algorithm~\ref{al: 1} gives the full procedure. We describe its two stages below.

\begin{algorithm}[ht]
\caption{CoE-Guided Coordination Heuristic}
\small
\begin{algorithmic}[1]
\REQUIRE Query $x$, model set $\mathcal{M} = \{\mathcal{M}_1, \dots, \mathcal{M}_K\}$,
         convergence threshold $\epsilon > 0$, maximum iterations $T_{\max}$
\ENSURE Per-model output distributions $\{p_1^*, \dots, p_K^*\}$,
        final CoE value $\mathcal{U}_{CoE}^*$

\STATE Set $t \leftarrow 0$
\STATE Initialise weights uniformly: $w_i^{(0)} \leftarrow \tfrac{1}{K}$ for all $i$
\FOR{$i = 1$ \TO $K$}
    \STATE Sample $m$ responses: $\mathcal{S}_i \leftarrow \{s_{i,1}, \dots, s_{i,m}\}$
    \STATE Cluster $\mathcal{S}_i$ via bidirectional entailment; compute $p_i^{(0)}$
           via Eq.~(\ref{eq: cluster prob})
\ENDFOR
\STATE Compute ensemble mean: $\bar{p}^{(0)} \leftarrow \sum_{i=1}^{K} w_i^{(0)} p_i^{(0)}$
\STATE Compute initial CoE: $\mathcal{U}_{CoE}^{(0)}$ via Eq.~(\ref{eq: objective})

\WHILE{$t < T_{\max}$}
    \FOR{$i = 1$ \TO $K$}
        \STATE \COMMENT{Step 2a: greedy Delta-distribution approximation}
        \STATE $j_i^* \leftarrow \arg\max_j\; p_i^{(t-1)}(c_j \mid x)$
        \STATE $p_i^{(t)}(c_j \mid x) \leftarrow \mathbf{1}[j = j_i^*]$
               \hfill\COMMENT{Eq.~(\ref{eq: greedy})}
        \STATE \COMMENT{Step 2b: entropy-proportional weight update}
        \STATE $\tilde{w}_i^{(t)} \leftarrow w_i^{(t-1)}
               \Bigl(1 + \sum_j p_i^{(t)}(c_j \mid x)\log p_i^{(t)}(c_j \mid x)\Bigr)$
    \ENDFOR
    \STATE Normalise: $w_i^{(t)} \leftarrow
           \tilde{w}_i^{(t)} \,/\, \sum_{k=1}^{K} \tilde{w}_k^{(t)}$
    \STATE Update ensemble mean: $\bar{p}^{(t)} \leftarrow \sum_{i=1}^{K} w_i^{(t)} p_i^{(t)}$
    \STATE Recompute: $\mathcal{U}_{CoE}^{(t)}$ via Eq.~(\ref{eq: objective})
    \IF{$\bigl|\mathcal{U}_{CoE}^{(t)} - \mathcal{U}_{CoE}^{(t-1)}\bigr| < \epsilon$}
        \STATE $p_i^* \leftarrow p_i^{(t)}$ for all $i$;\quad \textbf{break}
    \ENDIF
    \STATE $t \leftarrow t + 1$
\ENDWHILE
\STATE $\mathcal{U}_{CoE}^* \leftarrow \mathcal{U}_{CoE}^{(t)}$
\end{algorithmic}
\label{al: 1}
\end{algorithm}

\paragraph{Stage 1: Semantic clustering and initialisation (Lines 1--8).}
For each model $\mathcal{M}_i$, we draw $m$ candidate responses and partition them into
semantic clusters via bidirectional entailment scoring: two sequences $s_a$, $s_b$ are
assigned to the same cluster if each entails the other. The initial cluster probability
is:
\begin{equation}
\label{eq: cluster prob}
    p_i^{(0)}(c_j \mid x)
    \;=\; \frac{\displaystyle\sum_{s \in c_j} p(s \mid x)}
               {\displaystyle\sum_{k}\sum_{s \in c_k} p(s \mid x)},
\end{equation}
where $p(s \mid x)$ is the length-normalised generation log-probability of sequence $s$.
Ensemble weights are initialised uniformly: $w_i^{(0)} = 1/K$.

\paragraph{Stage 2: Iterative refinement (Lines 9--24).}

\noindent\textbf{Step 2a — Greedy Delta-distribution approximation (Lines 12--13).}
Motivated by Theorem~\ref{th: single-LLM negative entropy maximization} (i), each model's
distribution is updated toward its Delta distribution by concentrating all probability
mass on its currently most probable cluster:
\begin{equation}
\label{eq: greedy}
p_i^{(t)}(c_j \mid x) =
\begin{cases}
1 & \text{if } j = j_i^* := \arg\max_{j}\; p_i^{(t-1)}(c_j \mid x), \\
0 & \text{otherwise.}
\end{cases}
\end{equation}
This greedy step is $\mathcal{O}(l)$ per model per iteration and requires neither
gradient computation nor model retraining. Alternative solvers—gradient ascent on a
softmax-parameterised distribution, genetic algorithms, or
Bayesian optimisation over the cluster simplex—may be substituted at Line 11 when a
smoother optimisation trajectory is preferred.

\noindent\textbf{Step 2b — Entropy-proportional weight update (Lines 14--17).}
After updating the distributions, ensemble weights are adjusted to give higher influence
to models that have achieved greater semantic certainty:
\begin{equation}
\label{eq: w}
w_i^{(t)}
\;=\; \frac{
    w_i^{(t-1)}
    \Bigl(1 + \sum_j p_i^{(t)}(c_j \mid x)\log p_i^{(t)}(c_j \mid x)\Bigr)
}{
    \sum_{k=1}^{K} w_k^{(t-1)}
    \Bigl(1 + \sum_j p_k^{(t)}(c_j \mid x)\log p_k^{(t)}(c_j \mid x)\Bigr)
}.
\end{equation}
The numerator factor $1 + \sum_j p\log p$ equals $1 - SE_i^{(t)}$, which lies in
$[0, 1]$ and is maximised at $1$ when $p_i^{(t)}$ is a Delta distribution. The update
thus implements a soft version of majority-weight assignment: models that have already
converged to a confident cluster receive proportionally larger weight, pulling the
ensemble mean $\bar{p}^{(t)} = \sum_i w_i^{(t)} p_i^{(t)}$ toward the consensus.
The CoE is then recomputed via Eq.~(\ref{eq: objective}), and convergence is declared
when $|\mathcal{U}_{CoE}^{(t)} - \mathcal{U}_{CoE}^{(t-1)}| < \epsilon$.

\section{Complexity Analysis}
\label{se: Complexity Analysis}

The total computational cost of Algorithm~\ref{al: 1} decomposes into three stages.

\paragraph{Stage 1: Sequence generation.}
Generating $m$ candidate responses for each of the $\mathcal{K}$ models incurs
$\mathcal{O}\!\left(\sum_{i=1}^{\mathcal{K}}F(\mathcal{M}_i)\right)$,
where $F(\mathcal{M}_i)$ denotes the autoregressive inference cost of model $i$
(typically $\mathcal{O}(m\cdot L_{\max}\cdot d^2)$ for sequence length $L_{\max}$
and hidden dimension $d$). This is the dominant bottleneck.

\paragraph{Stage 2: Semantic clustering.}
Partitioning the $m\mathcal{K}$ responses into semantic clusters via pairwise bidirectional
entailment incurs $\mathcal{O}(\mathcal{K}\cdot\mathcal{C}_{\mathrm{cluster}})$, where
$\mathcal{C}_{\mathrm{cluster}}=\mathcal{O}(m^2\cdot C_{\mathrm{NLI}})$ and
$C_{\mathrm{NLI}}$ is the cost of a single NLI inference call.
In practice, $m$ is small ($m\leq 10$), so this term is modest relative to
Stage~1.

\paragraph{Stage 3: Iterative refinement.}
Each of the $T$ iterations performs a greedy Delta-distribution update
($\mathcal{O}(l)$ per model) and a CoE recomputation ($\mathcal{O}(\mathcal{K}\cdot l)$),
giving $\mathcal{O}(T\cdot \mathcal{K}\cdot l)$ in total.

The overall complexity is therefore:
\begin{equation}
    \mathcal{O}\!\left(
        \sum_{i=1}^{\mathcal{K}}F(\mathcal{M}_i)
        \;+\; \mathcal{K}\cdot\mathcal{C}_{\mathrm{cluster}}
        \;+\; T\cdot \mathcal{K}\cdot l
    \right).
    \label{eq: complexity}
\end{equation}

Since $l\leq 10$ and $T\leq 10$ in all our experiments, the $T\cdot \mathcal{K}\cdot l$ term
is negligible relative to the generation cost.
CoE thus adds minimal overhead to a standard multi-LLM inference pipeline:
Stages~2 and~3 are plug-and-play post-processing steps that operate on already-generated
outputs and do not require any model retraining or gradient computation.

\begin{table}[t]
\centering
\caption{The performance of different single LLMs on different dataset (200 samples)}
\label{tab:1}
\resizebox{0.5\linewidth}{!}{
\begin{tabular}{ccclcc} 
\toprule
\multirow{2}{*}{LLMs} & \multicolumn{2}{c}{\textit{TriviaQA} Dataset}                                                                            & \multicolumn{1}{c}{} & \multicolumn{2}{c}{\textit{SQuAD} Dataset}                                                                                \\ 
\cline{2-3}\cline{5-6}
                      & \begin{tabular}[c]{@{}c@{}}Train\\~Accuracy\end{tabular} & \begin{tabular}[c]{@{}c@{}}Validation \\Accuracy\end{tabular} &                      & \begin{tabular}[c]{@{}c@{}}Train \\Accuracy\end{tabular} & \begin{tabular}[c]{@{}c@{}}Validation \\Accuracy\end{tabular}  \\ 
\hline
Llama-3.1-8b-instruct & 0.65                                                     & 0.675                                                         &                      & 0.225                                                    & 0.1                                                            \\
Qwen-2.5-7b-instruct  & 0.6                                                      & 0.55                                                          &                      & 0.2                                                      & 0.15                                                           \\
Mistral-7b-instruct   & 0.375                                                    & 0.375                                                         &                      & 0.05                                                     & 0.075                                                          \\
\bottomrule
\end{tabular}}
\end{table}

\subsection{Detailed Experimental Results}
\label{sec: Experimental Results}

\textit{TriviaQA} is a large-scale reading comprehension dataset containing over 950K question-answer pairs authored by trivia enthusiasts and scraped from the Web, with distant supervision from 662K evidence documents. Its questions are designed to be challenging and compositional, requiring models to handle long, noisy contexts and complex reasoning. \textit{SQuAD} consists of 100K+ crowdsourced question-answer pairs on Wikipedia articles. The code will be released upon acceptance.

Table \ref{tab:1} presents the performance of three instruction-tuned large language models (Llama-3.1-8b-instruct, Qwen-2.5-7b-instruct, and Mistral-7b-instruct) on the \textit{TriviaQA} and \textit{SQuAD} machine reading comprehension datasets, evaluated using a 200-sample subset. Performance is quantified by accuracy on both training and validation splits, revealing that Llama-3.1-8b-instruct achieves the highest overall accuracy across both datasets: on \textit{TriviaQA}, it scores 0.65 on the training set and 0.675 on the validation set, while on \textit{SQuAD}, it attains 0.225 and 0.1 respectively. By contrast, Mistral-7b-instruct exhibits the lowest performance, with 0.375/0.375 accuracy on \textit{TriviaQA} and 0.05/0.075 on \textit{SQuAD}, and Qwen-2.5-7b-instruct yields intermediate results. These findings underscore the varying generalization capabilities of the models across different reading comprehension tasks, with Llama-3.1-8b-instruct demonstrating superior robustness, and all models showing a marked decline in accuracy when shifting from \textit{TriviaQA} to \textit{SQuAD}, suggesting dataset-specific task challenges.

\begin{figure*}[t]
    \centering
    \begin{minipage}{0.32\linewidth}
        \centering
        \centerline{
        \includegraphics[width=1\textwidth]{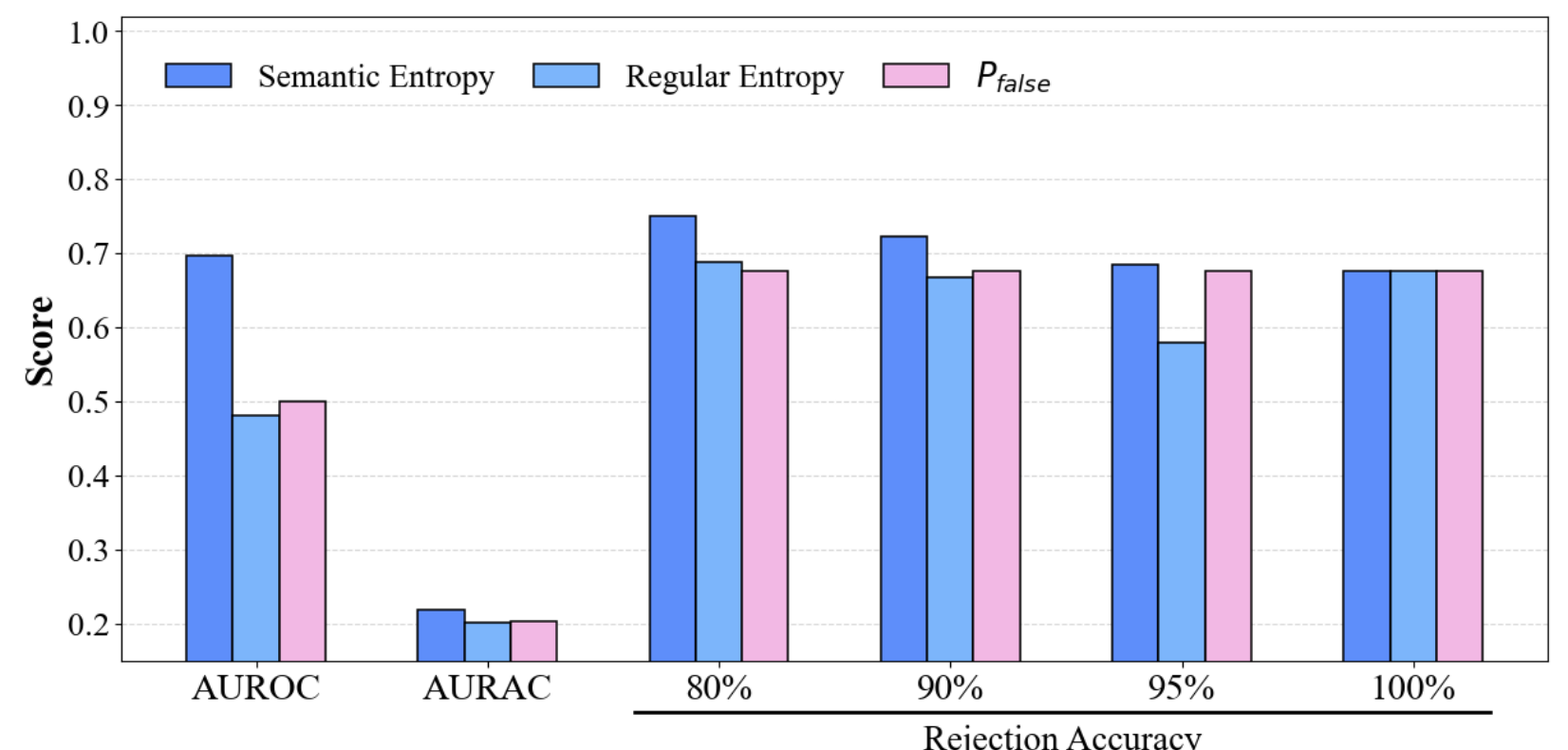}
        }
        \centerline{(\textit{a}) Llama}
    \end{minipage}
    \begin{minipage}{0.32\linewidth}
        \centering
        \centerline{
        \includegraphics[width=1\textwidth]{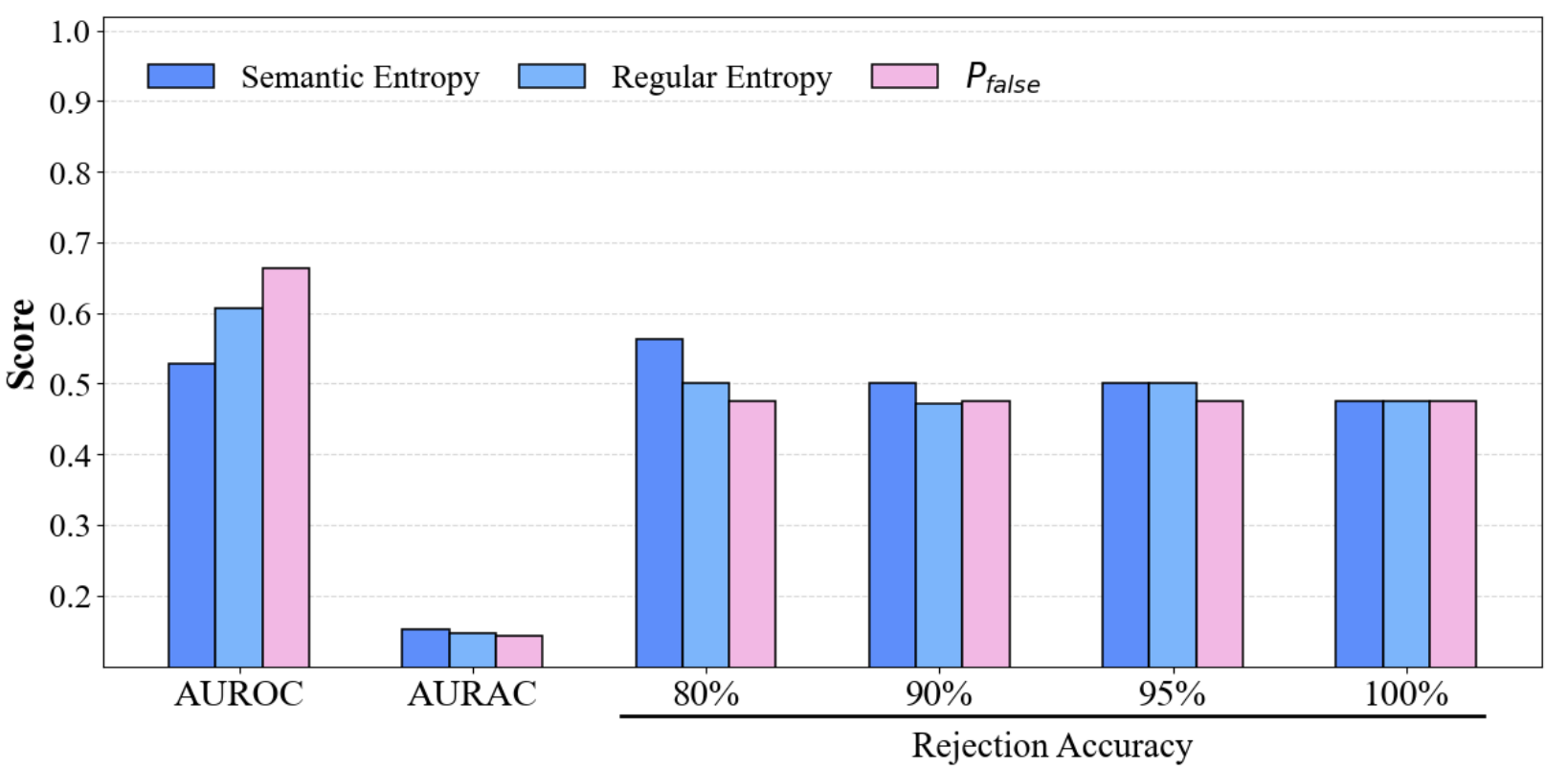}
        }
        \centerline{(\textit{b}) Qwen}
    \end{minipage}
    \begin{minipage}{0.32\linewidth}
        \centering
        \centerline{
        \includegraphics[width=1\textwidth]{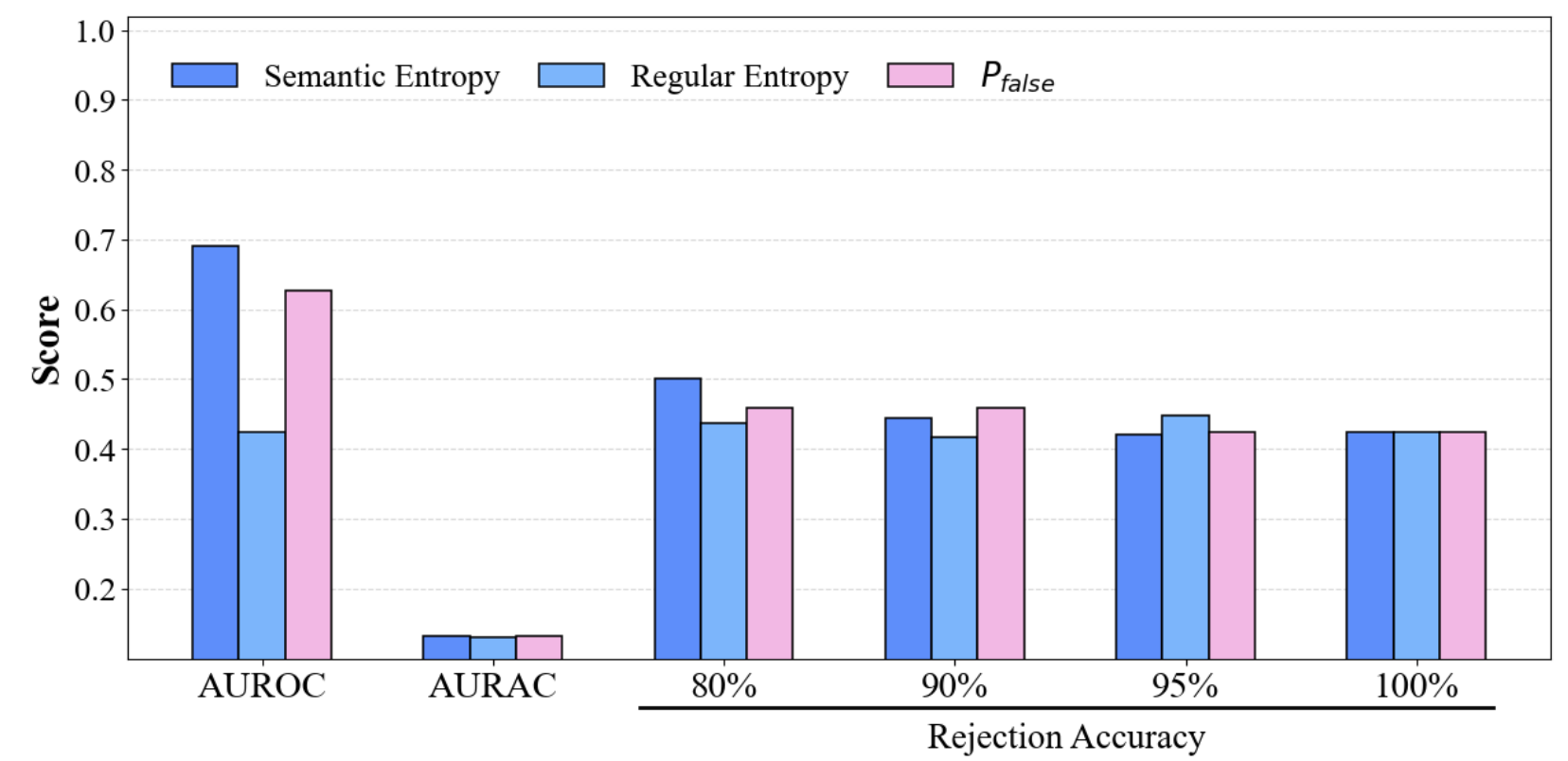}
        }
        \centerline{(\textit{c}) Mistral}
    \end{minipage}

    \begin{minipage}{0.32\linewidth}
        \centering
        \centerline{
        \includegraphics[width=1\textwidth]{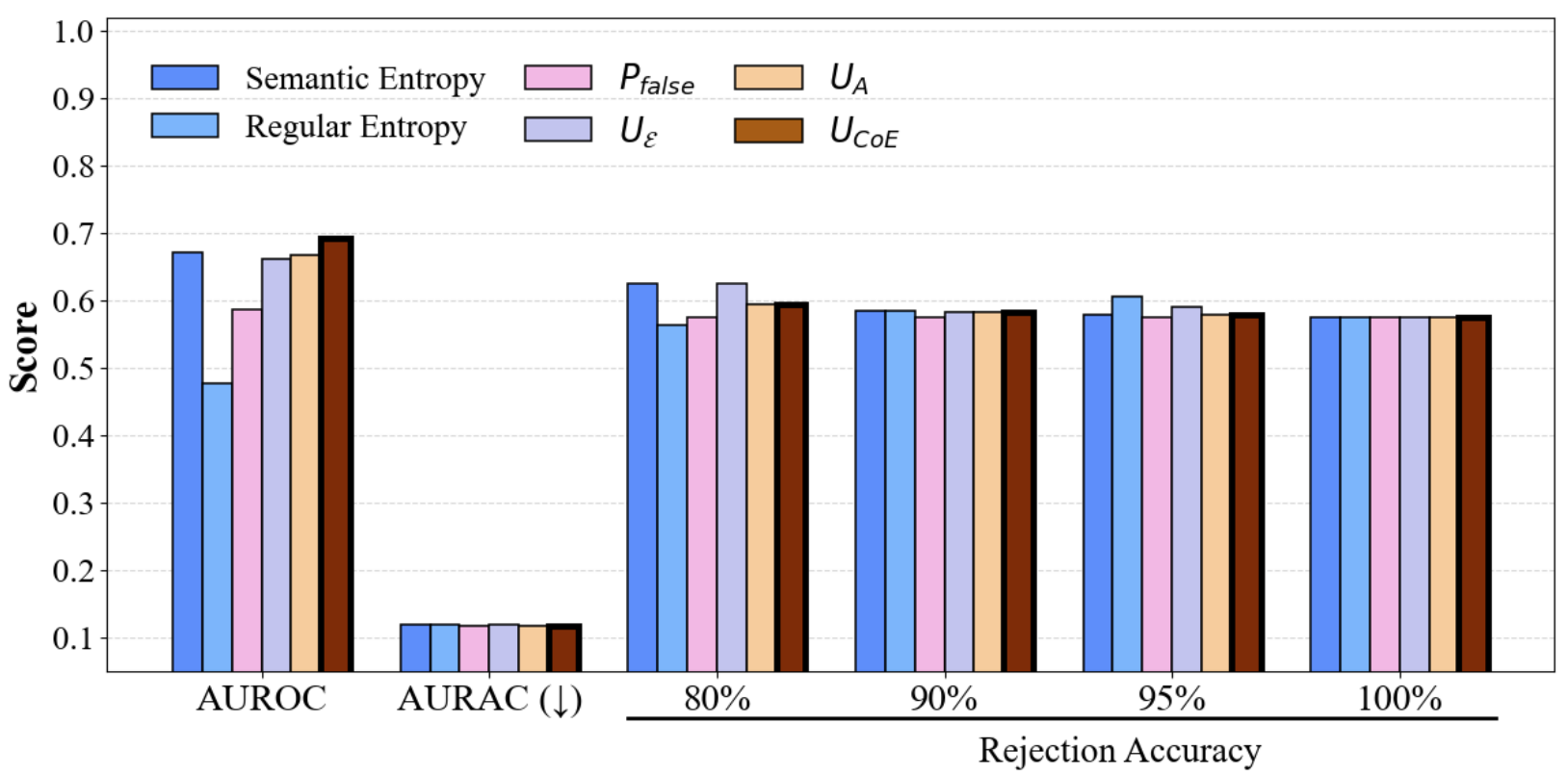}
        }
        \centerline{(\textit{d}) Llama + Qwen}
    \end{minipage}
    \begin{minipage}{0.32\linewidth}
        \centering
        \centerline{
        \includegraphics[width=1\textwidth]{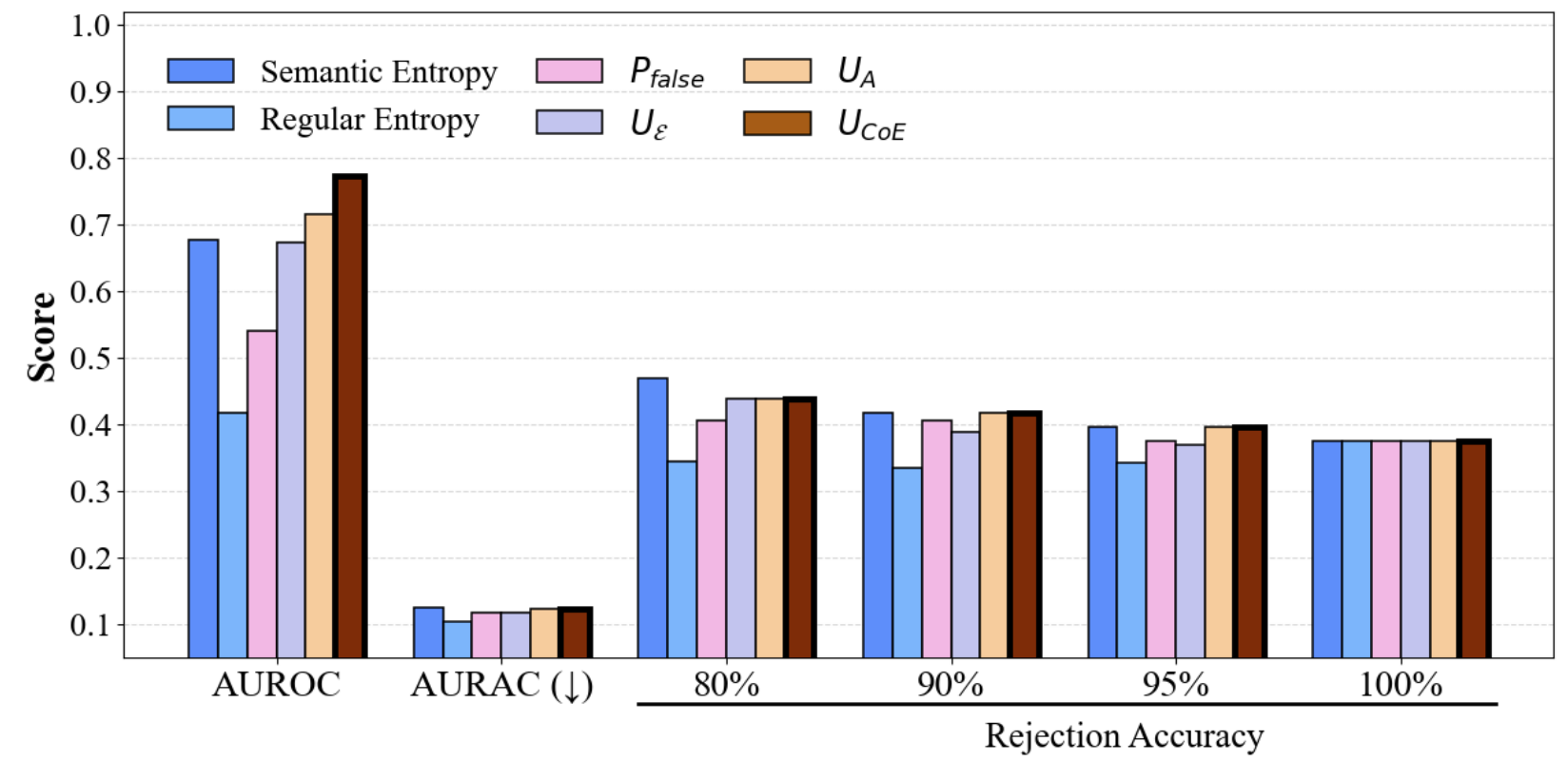}
        }
        \centerline{(\textit{e}) Llama + Qwen + Mistral}
    \end{minipage}
    \begin{minipage}{0.32\linewidth}
        \centering
        \centerline{
        \includegraphics[width=1\textwidth]{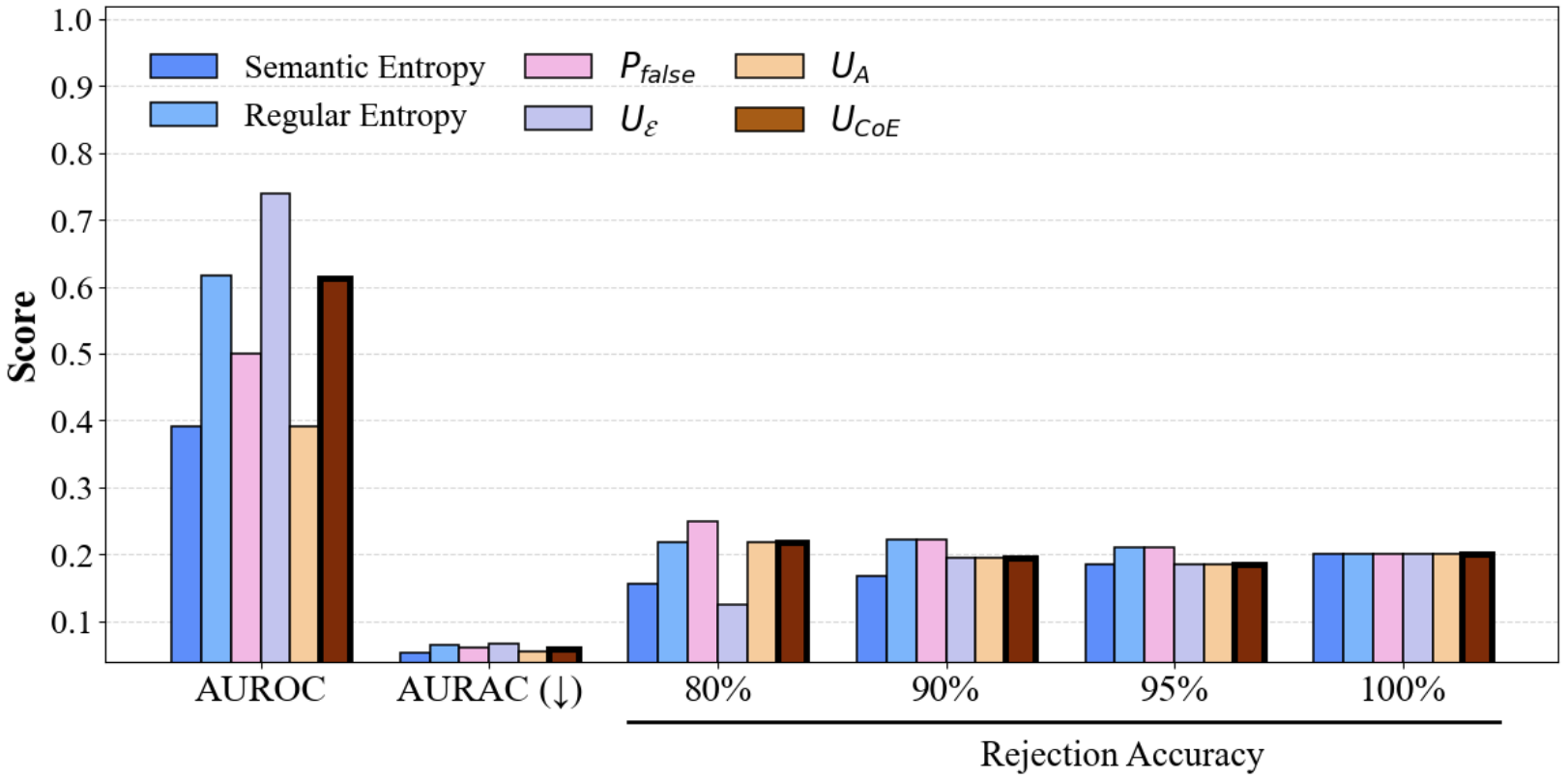}
        }
        \centerline{(\textit{f}) 2Llama + 2Qwen + 2Mistral}
    \end{minipage}
    \caption{Performance comparison of CoE with different divergence metrics on \textit{TriviaQA} dataset (200 samples, KL divergence).}
    \label{fig: 2}
\end{figure*}

As shown in Table~\ref{tab:2}, CoE consistently outperforms existing metrics across various ensemble configurations. In a two-LLMs' setup (Llama + Qwen), CoE achieves an AUROC of 0.683, significantly surpassing SE (0.670) and $\mathcal{U}_E$ (0.661). This superiority becomes more pronounced as the ensemble scales: with three models, CoE attains an AUROC of 0.772, a substantial gain over the strongest baseline $\mathcal{U}_E$ (0.716). These results, further visualized in Figure~\ref{fig: 2}, validate that CoE effectively harnesses complementary signals from heterogeneous models by jointly minimizing intra-model uncertainty and maximizing beneficial inter-model divergence.

\begin{figure*}[t]
    \centering
    \begin{minipage}{0.3\linewidth}
        \centering
        \centerline{
        \includegraphics[width=1\textwidth]{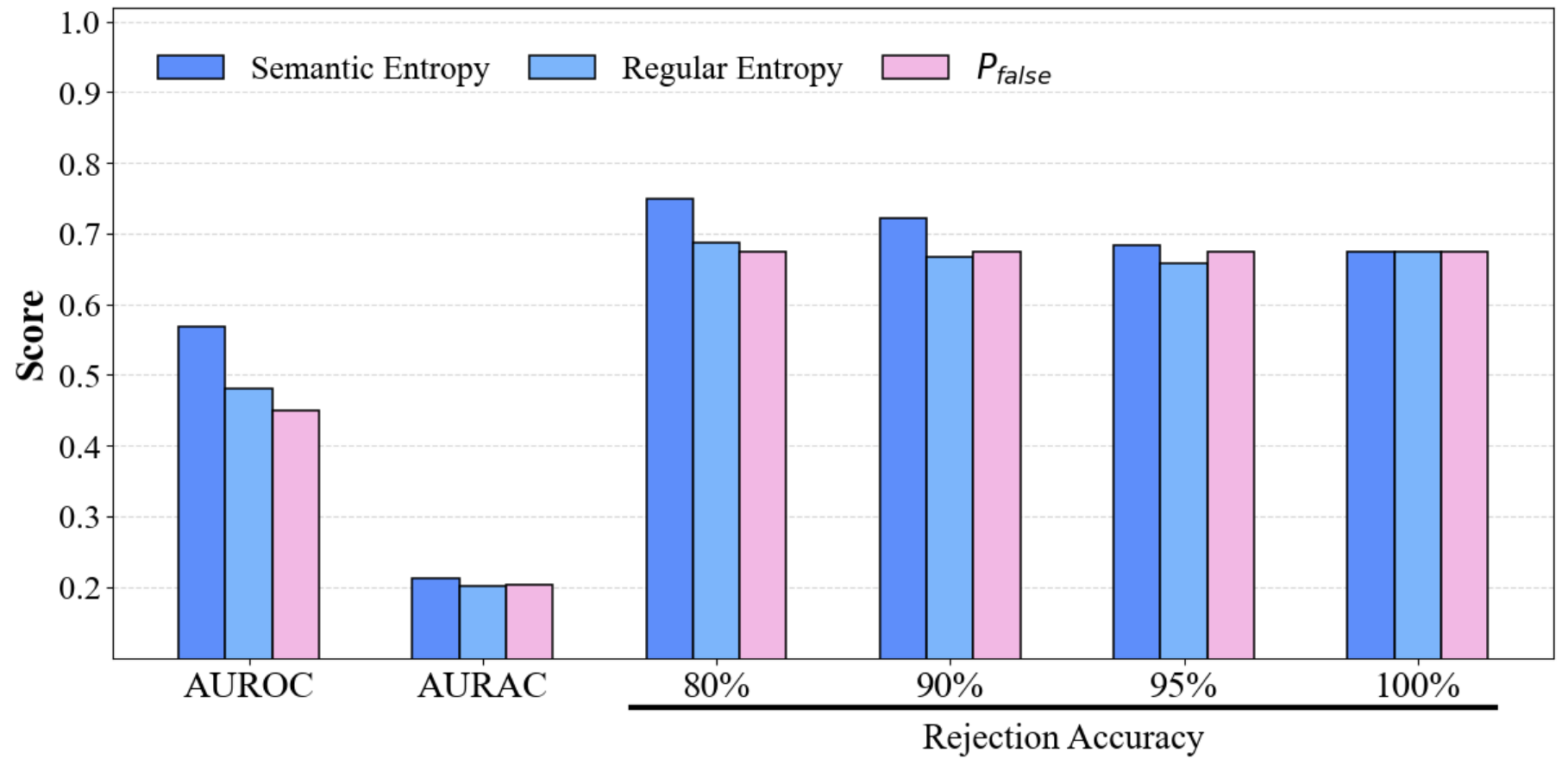}
        }
        \centerline{(\textit{a}) Llama}
    \end{minipage}
    \begin{minipage}{0.3\linewidth}
        \centering
        \centerline{
        \includegraphics[width=1\textwidth]{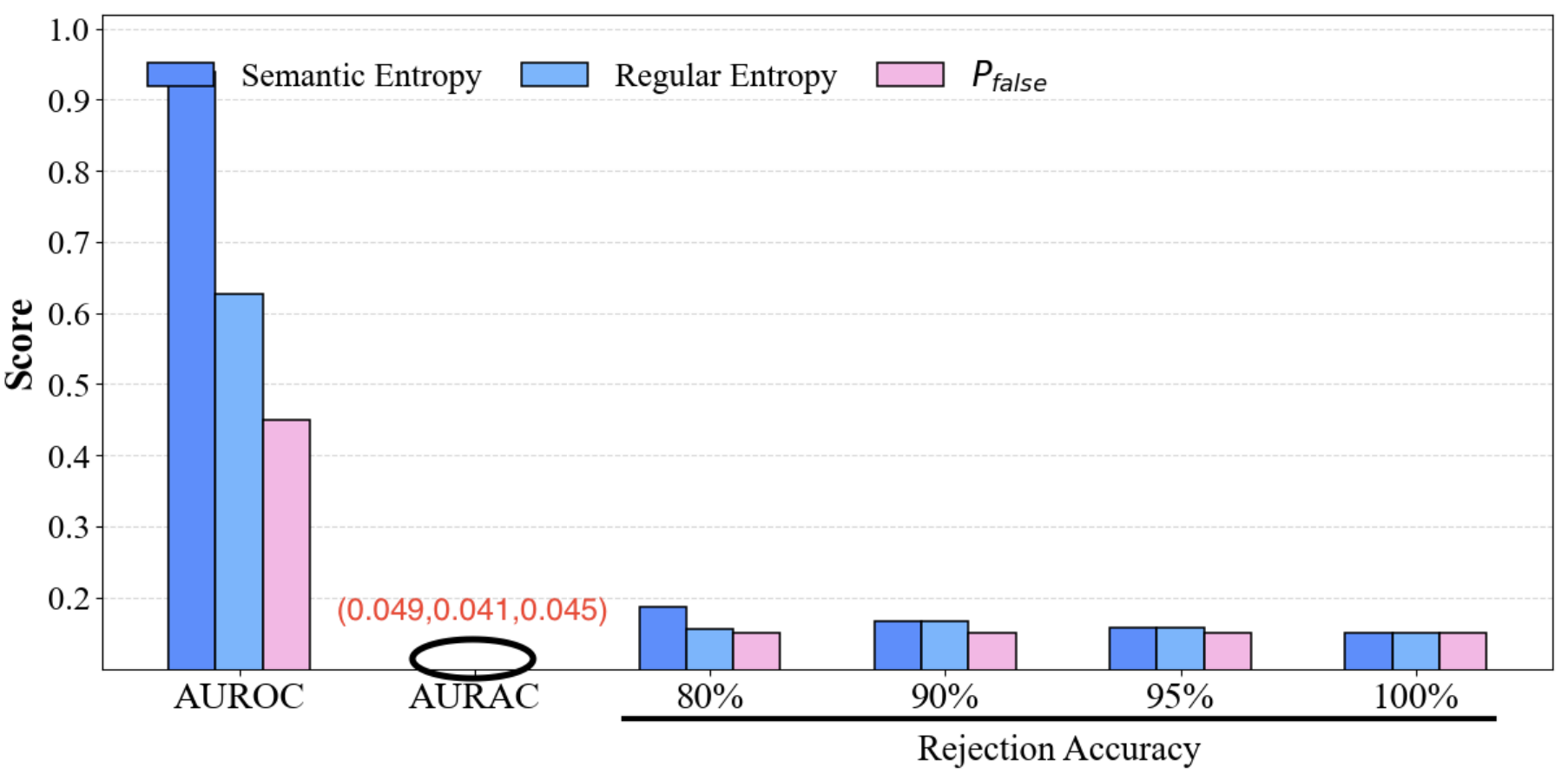}
        }
        \centerline{(\textit{b}) Qwen}
    \end{minipage}
    \begin{minipage}{0.3\linewidth}
        \centering
        \centerline{
        \includegraphics[width=1\textwidth]{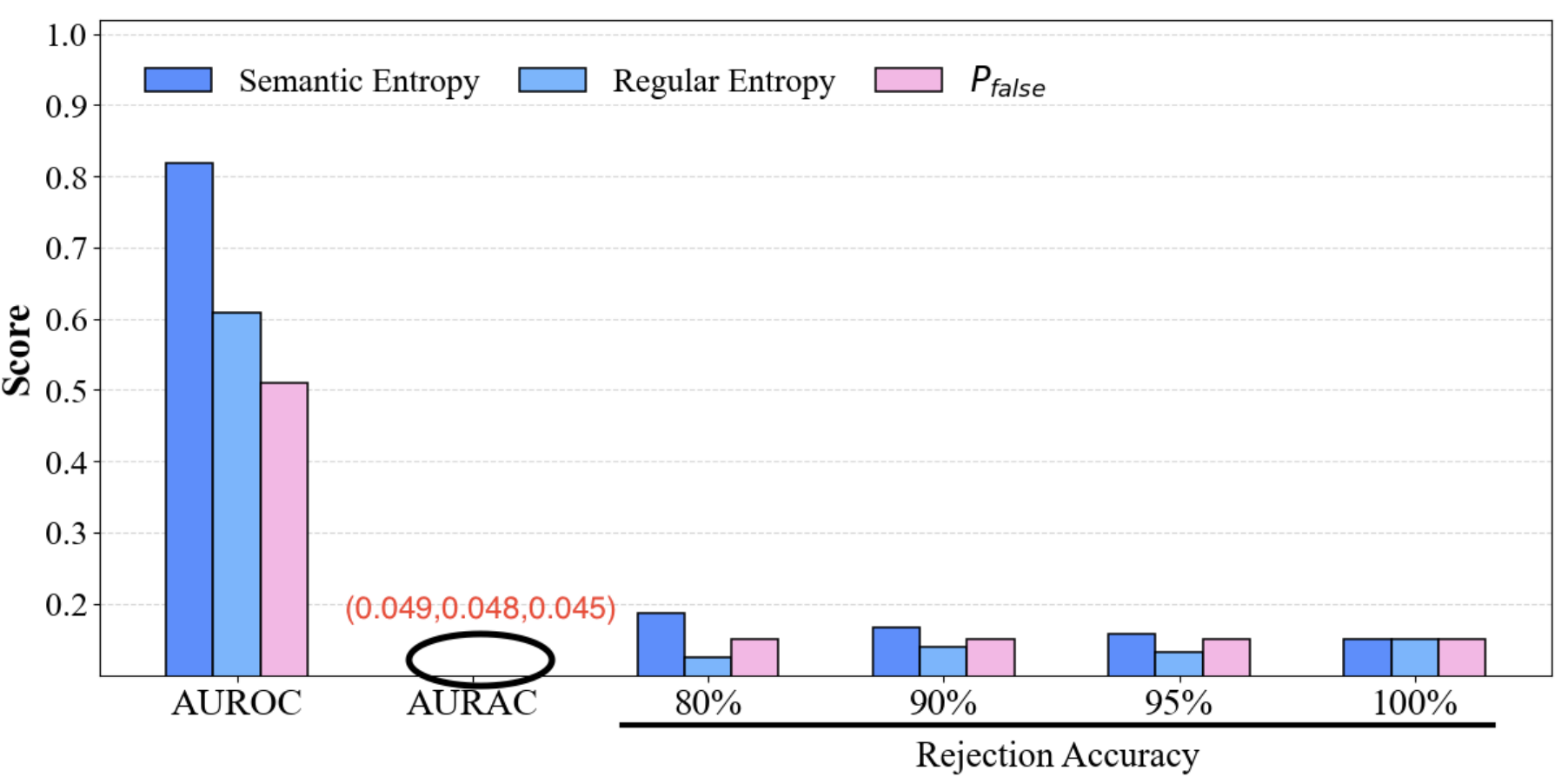}
        }
        \centerline{(\textit{c}) Mistral}
    \end{minipage}

    \begin{minipage}{0.3\linewidth}
        \centering
        \centerline{
        \includegraphics[width=1\textwidth]{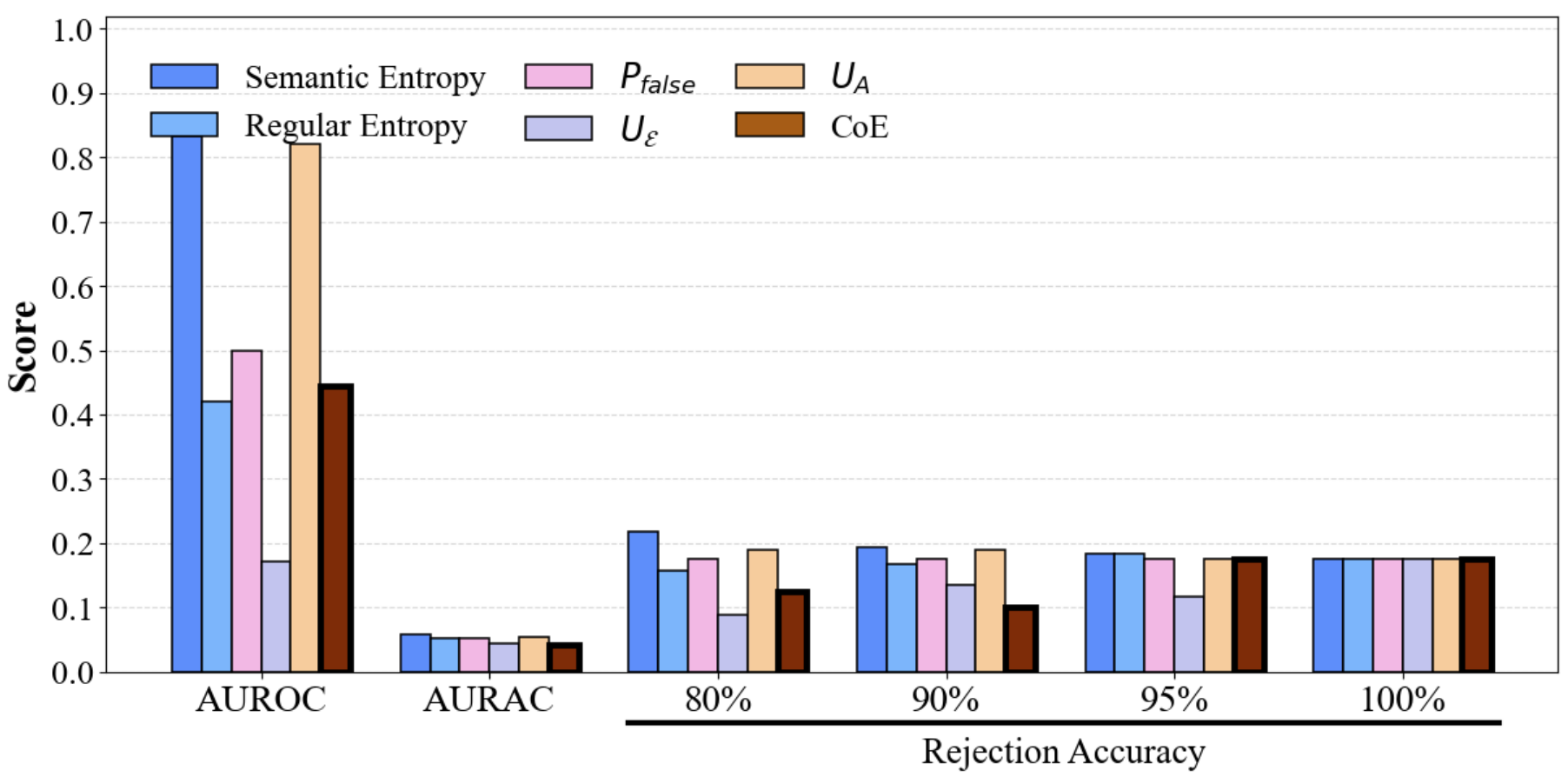}
        }
        \centerline{(\textit{d}) Llama + Qwen}
    \end{minipage}
    \begin{minipage}{0.3\linewidth}
        \centering
        \centerline{
        \includegraphics[width=1\textwidth]{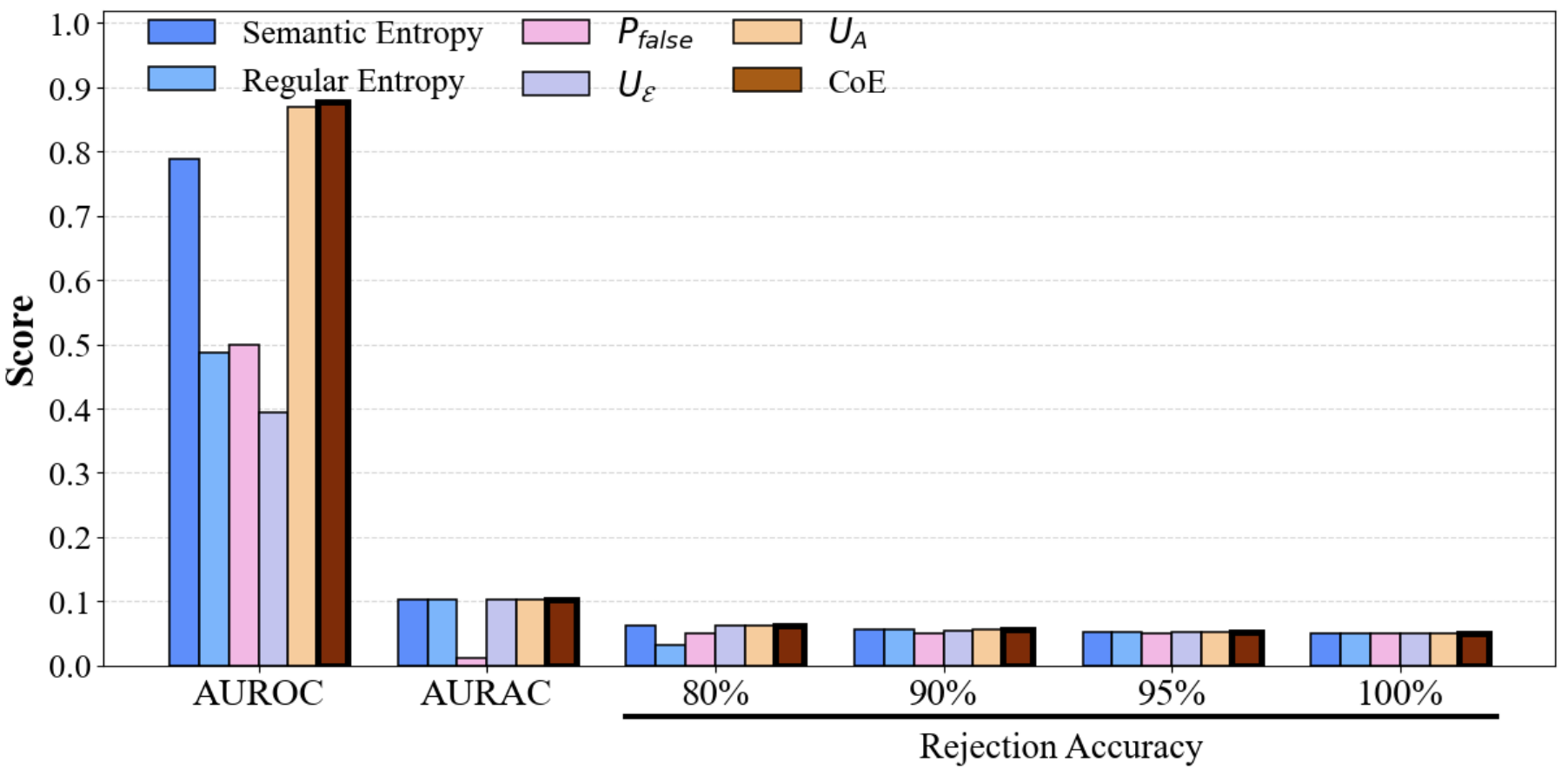}
        }
        \centerline{(\textit{e}) Llama + Qwen + Mistral}
    \end{minipage}
    \begin{minipage}{0.3\linewidth}
        \centering
        \centerline{
        \includegraphics[width=1\textwidth]{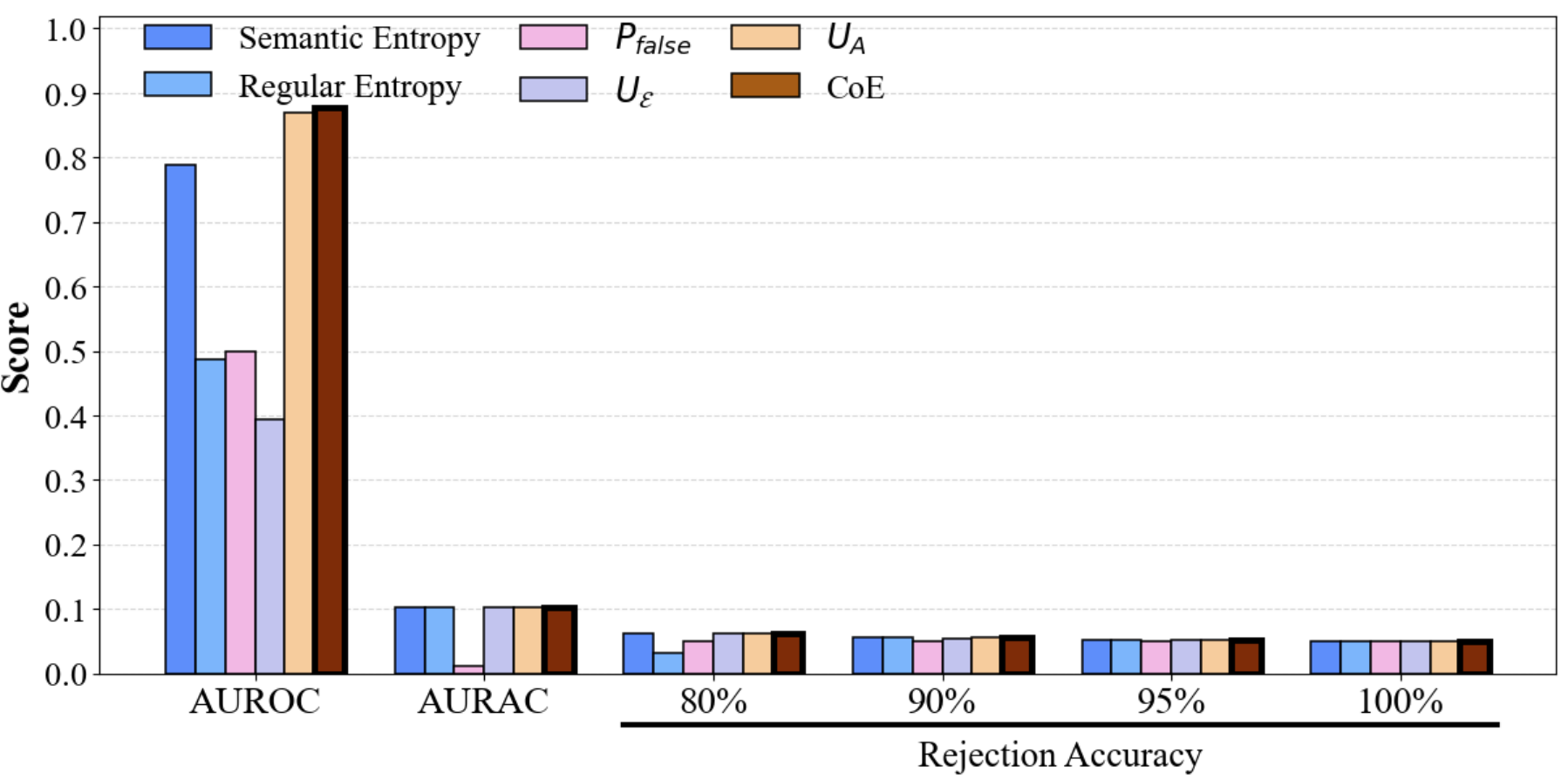}
        }
        \centerline{(\textit{f}) 2Llama + 2Qwen + 2Mistral}
    \end{minipage}
    \caption{Performance comparison of CoE with different divergence metrics on \textit{SQuAD} dataset (200 samples, KL divergence).}
    \label{fig: 3}
\end{figure*}
\begin{figure}[h]
    \centering
    \includegraphics[width=0.5\linewidth]{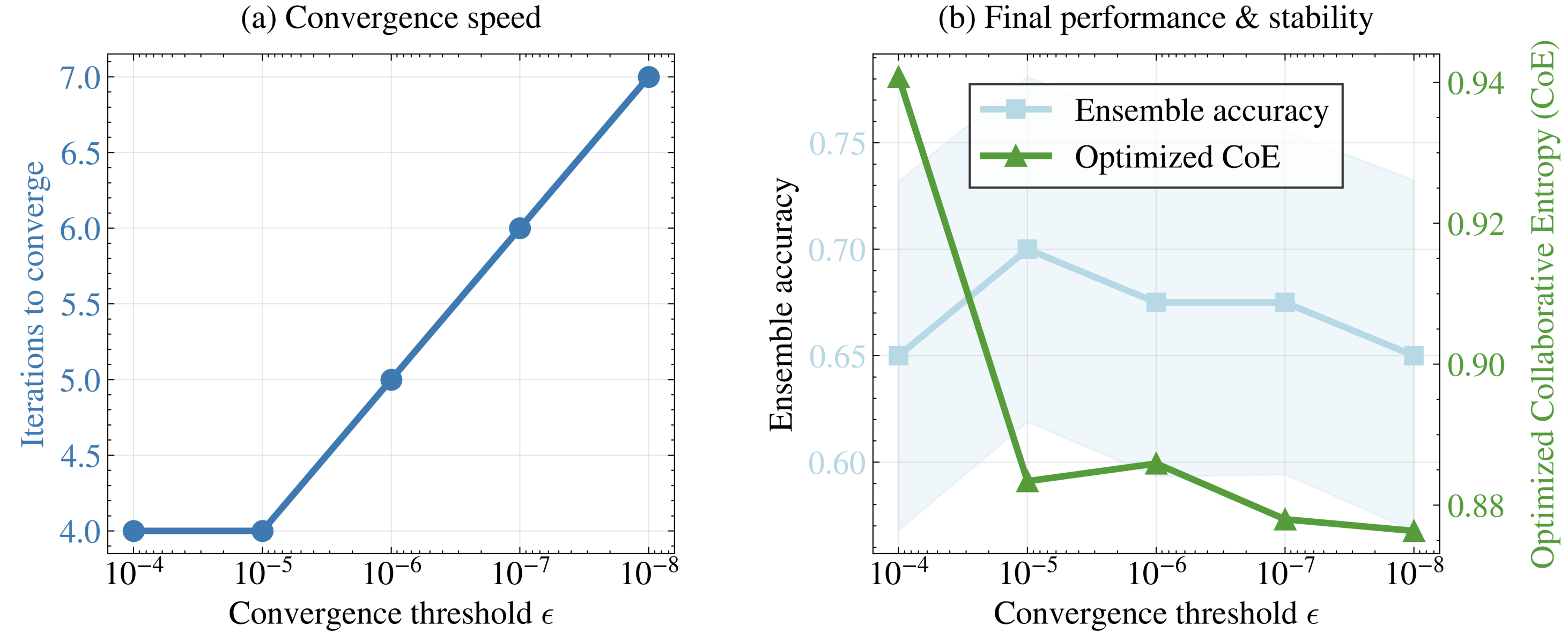}
    \caption{Robustness of the CoE-reduction algorithm to the convergence threshold $\epsilon$ (Llama + Qwen ensemble on \textit{TriviaQA}, 200 shots).}
    \label{fig: 4}
\end{figure}

\begin{figure}[t]
    \centering
    \begin{minipage}{0.3\linewidth}
            \centering
            \centerline{
            \includegraphics[width=1\textwidth]{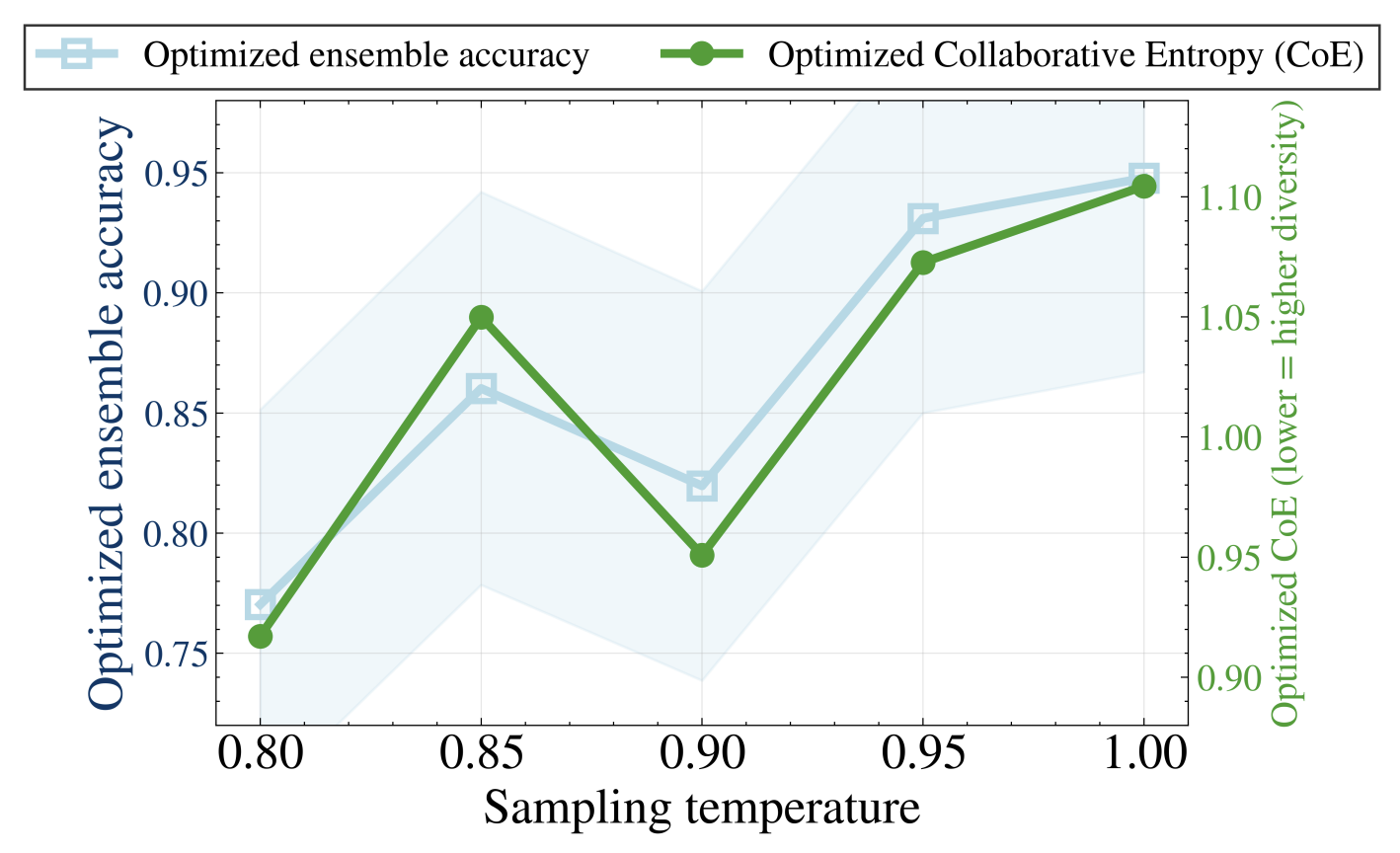}
            }
            \centerline{\small(\textit{a}) Different Temperatures }
    \end{minipage}
    \begin{minipage}{0.3\linewidth}
            \centering
            \centerline{
            \includegraphics[width=1\textwidth]{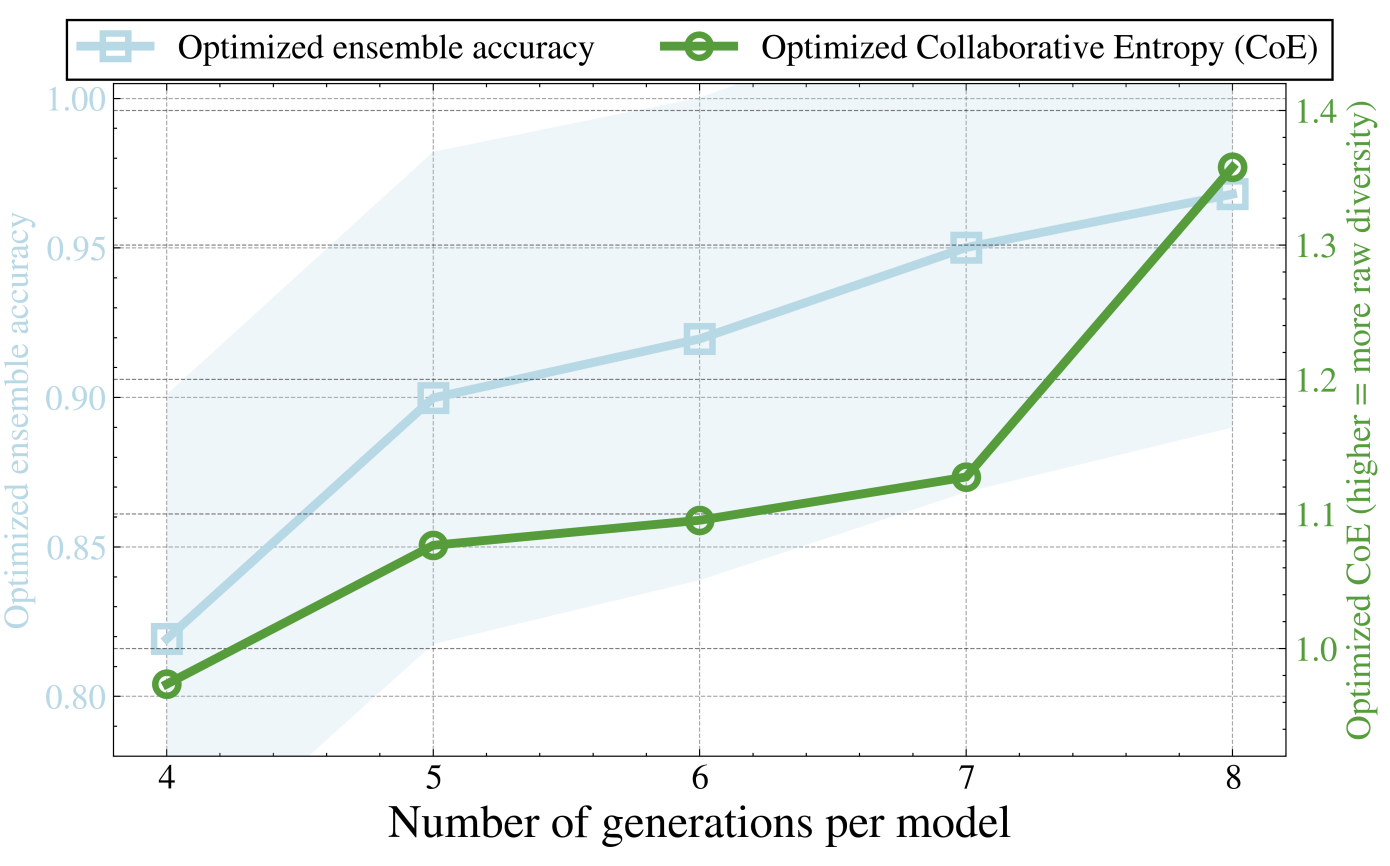}
            }
            \centerline{\small(\textit{b}) Different Generations }
    \end{minipage}
    \caption{Analysis of CoE-Reduction optimization algorithm under different sampling temperatures and generation counts (Llama + Qwen on \textit{TriviaQA} dataset with 200 in-context examples).}
    \label{fig: 5}
\end{figure}

Figure \ref{fig: 3} presents a parallel evaluation on the \textit{SQuAD} dataset, mirroring the structure of Figure \ref{fig: 2}. Subfigures (\textit{a})-(\textit{c}) assess the same individual models (Llama, Qwen, and Mistral), where Semantic Entropy again leads with the highest AUROC and rejection accuracy, though overall scores are lower than on \textit{TriviaQA} (e.g., $\text{AUROC} < 0.82$), reflecting \textit{SQuAD}'s emphasis on extractive reasoning and the models' varying capabilities in handling context-heavy tasks. The limited improvements in rejection accuracy (typically $< 0.2$ at 80\% retention) further expose single-model constraints in finer-grained uncertainty estimation. Subfigures (\textit{d})-(\textit{f}) highlight multi-LLM gains, with CoE's superiority scaling even more prominently on \textit{SQuAD}. In two-model (\textit{d}: Llama + Qwen), three-model (\textit{e}: Llama + Qwen + Mistral), and six-model (\textit{f}: 2Llama + 2Qwen + 2Mistral) setups, CoE consistently delivers the highest AUROC (reaching $0.878 \uparrow$ in three-model and $0.811 \uparrow$ in six-model configurations) and competitive rejection accuracy, surpassing baselines by up to 20\% in larger ensembles. This enhanced performance on \textit{SQuAD} validates CoE's robustness across datasets, as it effectively integrates inter-model divergence to mitigate context-specific hallucinations, outperforming methods like $\mathcal{U}_{\mathcal{A}}, \mathcal{U}_{\mathcal{E}}$, Semantic Entropy, Regular Entropy, and $P_{false}$. The results reinforce CoE's design efficacy in capturing collaborative signals, enabling more accurate UQ and selective prediction in diverse, real-world multi-LLM deployments for edge AI systems.

\paragraph{Convergence and Sensitivity} 

Figure \ref{fig: 4} evaluates the robustness of the CoE-reduction algorithm with respect to the convergence threshold \(\epsilon\) on the \textit{TriviaQA} dataset (200 shots, Llama + Qwen ensemble). In subplot (a), we observe that the number of iterations required for convergence increases linearly from approximately 4.0 to 7.0 as \(\epsilon\) decreases from \(10^{-4}\) to \(10^{-8}\), demonstrating that stricter convergence criteria impose higher computational costs but do not prevent the algorithm from converging. Subplot (b) reveals that the ensemble accuracy remains relatively stable, fluctuating between 0.65 and 0.70 with a peak near \(\epsilon = 10^{-6}\), indicating that the final performance of the ensemble is insensitive to the choice of \(\epsilon\). Meanwhile, the optimized CoE decreases from 0.94 to 0.88 as \(\epsilon\) becomes smaller, reflecting a trade-off between convergence stringency and the internal consistency of the ensemble. Together, these results confirm the robustness of the CoE-reduction algorithm: it maintains stable performance across varying convergence thresholds while exhibiting predictable convergence behavior.

\begin{figure}[t]
    \centering
    \includegraphics[width=0.52\linewidth]{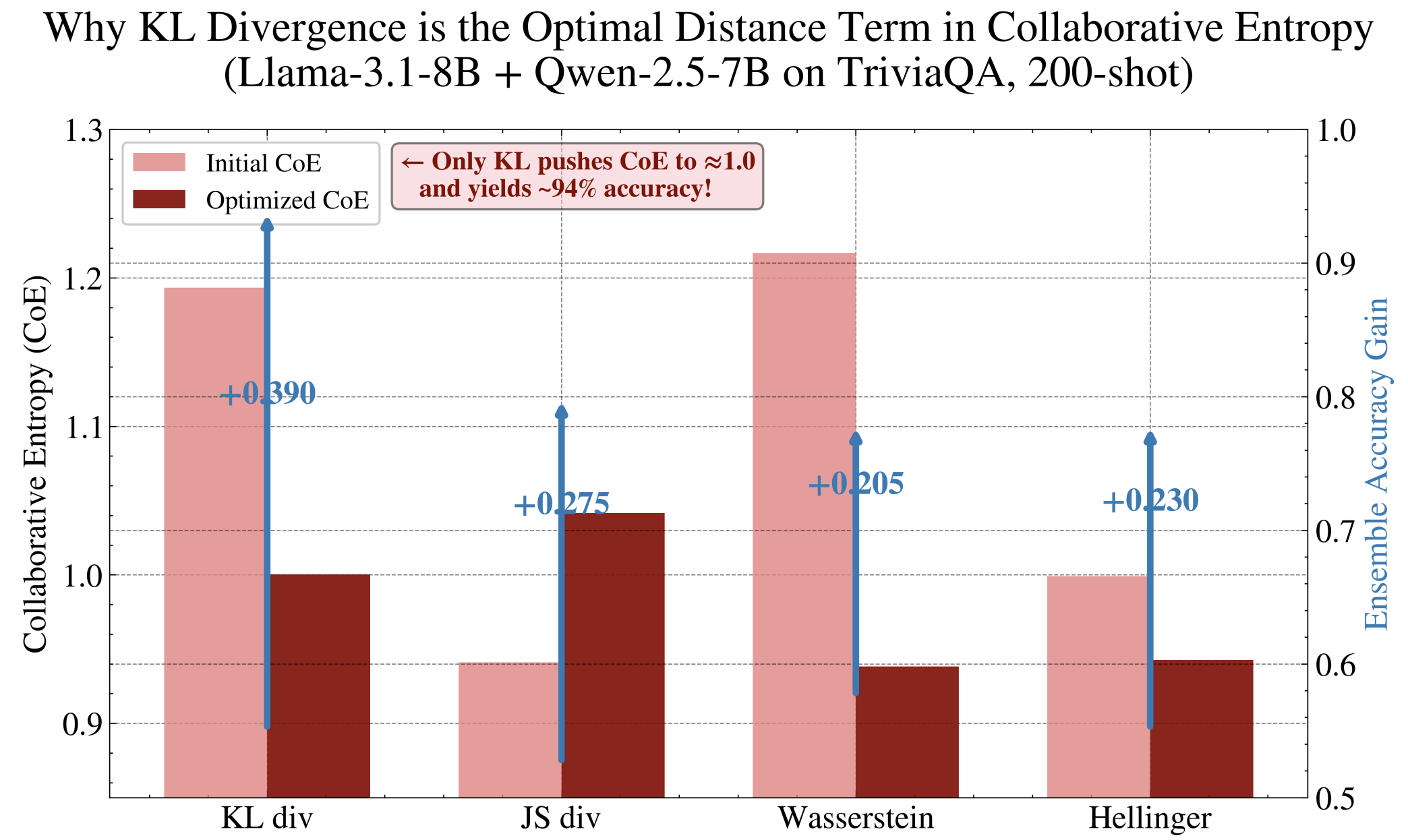}
    \caption{
        Accuracy gain and residual CoE after applying the CoE-guided coordination
        heuristic under four divergence measures (Llama + Qwen on TriviaQA,
        200 samples).
        Initial CoE (light bars) and optimised CoE (dark bars) are shown on the
        left axis; ensemble accuracy gain is annotated on each bar.
    }
    \label{fig: 6}
\end{figure}
\paragraph{Impact of Temperature and Sample Volume} 
Figure~\ref{fig: 5} demonstrates that CoE optimization is highly robust to sampling temperature. While higher temperatures ($0.8 \rightarrow 1.0$) naturally increase prediction diversity, CoE-guided weighting effectively transforms this stochasticity into useful signals, maintaining stable ensemble accuracy (92\%--95\%). Furthermore, performance scales gracefully with the number of generations (Figure~\ref{fig: 5}b); increasing samples from 4 to 8 per model yields a steady accuracy gain ($81.9\% \rightarrow 96.0\%$), indicating that CoE efficiently exploits available diversity even with low computational overhead.

\subsection{Analysis of the CoE-Guided Coordination Heuristic}
\label{se: heuristic eval}

\paragraph{Effect of divergence measure (Figure~\ref{fig: 6}).}

Figure~\ref{fig: 6} compares the four divergence instantiations of CoE on TriviaQA
(Llama + Qwen, 200 samples), reporting both initial and post-heuristic CoE values
and the resulting ensemble accuracy gain.
The asymmetric KL divergence is the only measure that drives CoE to approximately
its theoretical lower bound, yielding a \textbf{$+39.0\%$ accuracy gain}.
Symmetric alternatives achieve markedly smaller gains:
JS divergence ($+27.5\%$), Hellinger distance ($+23.0\%$), and
Wasserstein distance ($+20.5\%$) all leave a substantially larger residual CoE.

This result is consistent with the directional structure of inter-model epistemic
divergence.
$\mathcal{D}_{\mathrm{KL}}(p_i \| \bar{p})$ measures how far model $i$'s
distribution departs from the ensemble consensus—a naturally asymmetric quantity,
since the relevant question is whether an individual model is an outlier relative to
the group, not whether the group is an outlier relative to the individual.
Symmetric measures conflate these two directions and therefore underestimate the
epistemic divergence of confident but dissenting models.
We use KL divergence as the default in all CoE experiments.

\end{document}